\documentclass[opre,nonblindrev]{informs3aa}

\OneAndAHalfSpacedXI 

\RequirePackage{amsfonts}
\RequirePackage{natbib}
\RequirePackage[colorlinks,citecolor=blue,urlcolor=blue]{hyperref}
\usepackage{multirow}
\usepackage{enumitem}
\usepackage{xcolor}
\usepackage{subcaption}
\usepackage{multicol}
\usepackage{float}
\usepackage{epstopdf}
\usepackage[title]{appendix}

\newcommand{\tote}{\mathsf{TOTE}}

\newcommand{\bih}{\mathsf{BIH}}

\usepackage{endnotes}

\usepackage{natbib}
\bibpunct[, ]{(}{)}{,}{a}{}{,}%
\def\bibfont{\small}%
\TheoremsNumberedThrough     
\ECRepeatTheorems

\EquationsNumberedThrough    

\let\oldnl\nl
\newcommand{\nonl}{\renewcommand{\nl}{\let\nl\oldnl}}

\newcommand{\vct}{\boldsymbol }

\newcommand{\ud}{\mathrm d}

\newcommand{\kl}{\mathrm{KL}}

\newcommand{\poly}{\mathrm{poly}}

\renewcommand{\hat}{\widehat}
\renewcommand{\tilde}{\widetilde}
\newcommand{\sR}{\mathsf R}

\usepackage{algorithm}
\usepackage{algpseudocode}

\renewenvironment{proof}{\noindent {\it {Proof.} }}{\hfill $\Box$ \\}

\usepackage{xcolor}
\newcount\Comments  
\definecolor{darkgreen}{rgb}{0,0.5,0}
\definecolor{darkred}{rgb}{0.7,0,0}
\definecolor{teal}{rgb}{0.3,0.8,0.8}
\newcommand{\kibitz}[2]{\ifnum\Comments=1\textcolor{#1}{\small #2}\fi}

\newcommand{\xichen}[1]{\kibitz{darkred}{[XC: #1]}}

\newtheorem{fact}{Fact}

\renewcommand{\hat}{\widehat}
\renewcommand{\tilde}{\widetilde}
\renewcommand{\bar}{\overline}

\newcommand{\red}{\color{red}}
\newcommand{\blue}{}

\definecolor{DSgray}{cmyk}{0,1,0,0}

\begin{document}


\RUNAUTHOR{Chen, Krishnamurthy and Wang}

\RUNTITLE{Robust Dynamic Assortment Optimization}

\TITLE{Robust Dynamic Assortment Optimization in the Presence of Outlier Customers}

\ARTICLEAUTHORS{%
\AUTHOR{Xi Chen}
\AFF{Stern School of Business, New York University, New York, NY 10012, \EMAIL{xchen3@stern.nyu.edu}} 
\AUTHOR{Akshay Krishnamurthy}
\AFF{Microsoft Research NYC, New York, NY 10011, \EMAIL{akshay.krishnamurthy@microsoft.com}}
\AUTHOR{Yining Wang}
\AFF{Warrington College of Business, University of Florida, Gainesville, FL 32611, \EMAIL{yining.wang@warrington.ufl.edu}}
} 

\ABSTRACT{%
We consider the dynamic assortment optimization problem under the multinomial logit model (MNL) with unknown utility parameters. 
The main question investigated in this paper is model mis-specification under the $\varepsilon$-contamination model, which is a fundamental model in robust statistics and machine learning. 
In particular, throughout a selling horizon of length $T$, we assume that customers make purchases according to a well specified underlying multinomial logit choice model in a $(1-\varepsilon)$-fraction of the time periods, and make arbitrary purchasing decisions instead in the remaining $\varepsilon$-fraction of the time periods. 
In this model, we develop a new robust online assortment optimization policy via an active elimination strategy. 
We establish both upper and lower bounds on the regret, and show that our policy is optimal up to logarithmic factor in $T$ when the assortment capacity is constant. 
We further develop a fully adaptive policy that does not require any prior knowledge of the contamination parameter $\varepsilon$. 
{
	In the case of the existence a sub-optimality gap between optimal and sub-optimal products,
	we also established gap-dependent logarithmic regret upper bounds and lower bounds in both the known-$\varepsilon$
	and unknown-$\varepsilon$ cases.
} Our simulation study shows that our policy outperforms the existing policies based on upper confidence bounds (UCB) and Thompson sampling.
}



\maketitle

%


\section{Introduction}

A wide range of operations problems, ranging from assortment optimization to supply chain management, are built  on an underlying probabilistic model.
When real world outcomes follow this model, existing optimization techniques are able to provide  accurate solutions. 
However, these model assumptions are only abstractions of reality and do not perfectly capture the sophisticated natural environment. 
In other words, these models are inherently mis-specified to a certain degree. 
Accordingly, model mis-specification and robust estimation have been an important topic in the statistics literature~\citep{Huber:11}. 
However, this literature primarily focuses on estimation or prediction from a given dataset, which is insufficient for modern operations settings where decision making plays a vital role. 
Unfortunately, most decision-making policies are derived from optimization problems that explicitly rely on the probabilistic model, so they are inherently not robust to model mis-specification.
Can we design robust policies for these operations problems?


This paper studies model mis-specification for an important problem in revenue management --- dynamic assortment optimization,  under a popular $\varepsilon$-contamination model (which will be introduced in the next paragraph). 
Assortment optimization has a wide range of applications in retailing and online advertising.
Given a large number of substitutable products, the assortment optimization problem involves selecting a subset of products (a.k.a., an assortment) to offer a customer such that the expected revenue is maximized. 
To model customers' choice behavior when facing a set of offered products, discrete choice models have been widely used, and one of the most popular such models is the \emph{multinomial logit model (MNL)} \citep{McFadden1974}.
In dynamic assortment optimization, the customers' choice behavior (e.g., mean utilities of products in an MNL) is not known \emph{a priori} and must be learned online, which is often the case in practice, as historical data is often insufficient (e.g., fast fashion sale or online advertising).
More specifically, the seller offers an assortment to each arriving customer for a finite time horizon $T$, observes the purchase behavior of the customer and then updates the utility estimate.
The goal of the seller is to maximize the cumulative expected revenue over $T$ periods.
Due to its practical relevance, dynamic assortment optimization has received much attention in literature. \citep{Caro2007, Rusmevichientong2010, Saure2013, Agrawal17MNLBandit, Agrawal17Thompson}. 


All of these existing works assume that each arriving customer makes her purchase according to an underlying choice model. 
Yet, in practice, a small fraction of  customers could make ``outlier'' purchases. 
To model such outlier purchases, we adopt a natural robust model in the statistical literature ---  the $\varepsilon$-contamination model  \citep{Huber:64}, which dates back to the 1960s and is perhaps the most widely used model in robust statistics. 
In the general setup of the $\varepsilon$-contamination model, we are given $n$ \emph{i.i.d.} samples drawn from a distribution $(1-\varepsilon)P_{\theta} + \varepsilon Q$, where $P_{\theta}$ denotes the distribution of interest $P$, parameterized by $\theta$ (e.g., a Gaussian distribution with mean $\theta$), and $Q$ is an arbitrary contamination distribution.
The parameter $\varepsilon>0$, which is usually very small, reflects the level at which contamination occurs, so a larger $\varepsilon$ value means more observations are contaminated.
The standard objective is to identify or estimate the parameter $\theta$ of the distribution of interest, in the presence of corrupted observations from $Q$.
For the purpose of dynamic assortment optimization in the presence of outlier customers, the $P_\theta$ distribution represents the choice model for the majority of customers, which are ``typical,'' (with $\theta$ being the parameter of an underlying
MNL choice model of interest), while the $Q$ distribution corresponds to choice models of ``outlier'' customers and $\varepsilon$ reflects the proportion of outlier customers. 
For dynamic assortment optimization, we also deviate from the standard parameter estimation objective and focus on designing online decision-making policies.

In the classical $\varepsilon$-contamination model, the ``outlier distribution'' $Q$ stays stationary for all samples,
To make the contamination model more practical in the online assortment optimization setting,
we strengthen the model from two aspects:
\begin{enumerate}
	\item Instead of assuming  a fixed corruption distribution $Q$ for all outlier customers, we allow $Q$ to change over different time periods (i.e., $Q_t$ is the outlier distribution for customers at time period $t$); 
	\item  {Instead of assuming that each time $t$ is corrupted ``uniformly at random'', we assume that outlier customers appear in at most $\varepsilon T$ time periods. 
      The purchase pattern and arrivals of outlier customers can, however, be arbitrary and even \emph{adaptive} to the assortment decisions or customer purchase activities prior to time period $t$. 
	  The corrupted time periods and associated $Q_t$'s are unknown to the seller. 
    }
\end{enumerate}
This setting is much richer than the ``random arrival setting'' and more realistic in practice.
Indeed, in a holiday season, consecutive time periods might contain anomalous or outlier purchasing behavior, which cannot be capture by ``random corruption'' in the original $\varepsilon$-contamination model.
The details of our outlier customer model will be rigorously specified in Section \ref{sec:formulation}. 

The main goal of the paper is to develop a robust dynamic assortment policy under this $\varepsilon$-contaminated MNL.
Our first observation is that popular policies in the literature including Upper-Confidence-Bounds (UCB) \citep{Agrawal17MNLBandit} and Thompson sampling \citep{Agrawal17Thompson} no longer work in this model.
The reason is that these policies cannot use typical customers that arrive later in the selling period to correct for misleading customers that arrive early on, and hence even a small number of outlier customers can lead to poor performance. 
Further, while it is well known that \emph{randomization} is crucial in any adversarial setting (see, e.g.,~\cite{auer2002nonstochastic,Bubeck:Survey:12}) to hedge against outliers, UCB is a deterministic policy, while Thompson sampling provides very little randomization via posterior sampling.
We explain these failures in more detail in Secs.~\ref{sec:formulation} and~\ref{sec:numerical} later in this paper.


To address the contaminated setting,  we develop a novel active elimination algorithm for robust dynamic planning, which gradually eliminates those items that are not in the optimal assortment with high probability (see Algorithm \ref{alg:active-elimination}).  
Compared to the existing methods mentioned above~\citep{Agrawal17MNLBandit,Agrawal17Thompson}, our active elimination method has several important technical novelties.
First, our active elimination policy implements the randomization in a much more explicit way by sampling from a carefully constructed small set of ``active'' products.
Second, the existing UCB and Thompson sampling algorithms for MNL rely on an epoch-based strategy (i.e., repeatedly offering the same assortment until no purchase) to enable an unbiased estimation of utility parameters. 
This procedure is inherently fragile since the stopping time of an epoch relies on a single no-purchase activity, which can be easily manipulated by outlier customers; a few outliers can greatly affect the stopping times. 
The failure of such an epoch-based strategy implies that unbiased estimation of utility parameters is no longer possible. 
To overcome this challenge, we propose a new utility estimation strategy based on geometrically increasing offering time periods.
We conduct a careful perturbation analysis to control the bias of these estimates, which leads to new confidence bounds for our active elimination algorithm (see Sec.~\ref{sec:act_policy} for more details).

We provide theoretical guarantees for our proposed robust policy via regret analysis and information-theoretic lower bounds.
In particular, let $T$ be the selling horizon, $N$ the total number of products, and $K$ the cardinality constraint of an assortment (see Sec.~\ref{sec:formulation}).
For the reasonable setting where $\varepsilon$ is not too large, our active elimination algorithm (Algorithm \ref{alg:active-elimination}) achieves $\tilde{O}(\varepsilon K^2T +\sqrt{KNT})$ regret when $\varepsilon$ (or a reasonable upper bound of $\varepsilon$) is known (see Theorem \ref{thm:upper-bound}), where $\tilde{O}(\cdot)$ only suppresses $\log(T)$ factors. 
Compared to the $\Omega(\varepsilon T+\sqrt{NT})$ lower bound (see Proposition \ref{prop:lowerbound}), our upper bound is tight up to polynomial factors involving $K$ and other logarithmic factors.
We also remark that the special case of $\varepsilon=0$ reduces to the existing setting studied in \citep{Agrawal17MNLBandit,Agrawal17Thompson,Chen:18tight} in which no outlier customers are present.
Compared to existing results, our regret bound is tight except for an additional $O(\sqrt{K})$ factor, which represents the cost of being adaptive to outlier customers (see Sec.~\ref{sec:regret} for more discussions). 
{We emphasize that in a typical assortment optimization problem, the capacity of an assortment $K$ is usually a small constant, especially relative to $T$ and $N$. }


The above result assumes that an upper bound on the outlier proportion $\varepsilon$ is given as prior knowledge.
While in some cases we may be able to estimate $\varepsilon$ from historical data, this is not always possible, which motivates the design of fully adaptive policies that do not require $\varepsilon$ as an input. 
Inspired by the  ``multi-layer active arm race'' from the multi-armed bandits literature \citep{Lykouris:18}, we propose an adaptive robust dynamic assortment optimization policy in Algorithm \ref{alg:adaptive-epsilon}. 
Our policy runs multiple ``threads'' of known-$\varepsilon$ algorithms on a geometric grid of $\varepsilon$ values in parallel, and, as we show, achieves $\tilde{O}(\varepsilon T +\sqrt{NT})$ regret, where $\tilde{O}$ suppresses $\log(T)$ and $K$ factors (see Theorem \ref{thm:adaptive-epsilon}).
{Here, the (cumulative) regret is defined as the sum of the differences between the expected rewards (revenues)
of the optimal assortment and the assortments the retailer offers at each time period.}
Algorithm \ref{alg:adaptive-epsilon} and its analysis in Sec.~\ref{sec:adaptive} provide more details.

{
Finally, in the case of well-separated problem instances (i.e. there is a large sub-optimality gap $\beta>0$ between optimal and sub-optimal assortments),
built on the same proposed algorithm, we establish much improved regret upper bounds of $\widetilde O(\varepsilon K^2 T\log T + K^2 N\log^2 T/\beta)$
when $\varepsilon$ is known (see Theorem  \ref{thm:known-eps-gap-dependent}).  When $\varepsilon$ is unknown, the adaptive policy achieves the regret $\widetilde O(\varepsilon T + N/\beta^2)\times \poly(K,\log(NT))$ or $\widetilde O(\varepsilon T/\beta + N/\beta)\times\poly(K,\log(NT))$,
whichever is smaller (see Theorem \ref{thm:adaptive-epsilon-gap}). 
For both upper bounds in the well-separated case, the dependency on the time horizon $T$ is logarithmic when the corruption level $\varepsilon$ is small.
We also prove lower bounds on the regret when a sub-optimality gap of at least $\beta>0$ exists.
}

The rest of the paper is organized as follows.
Sec. \ref{sec:rel} introduces the related work.
Sec. \ref{sec:formulation} describes the problem formulation.
The first active elimination policy and the regret bounds are presented in Sec. \ref{sec:act_policy}, while the adaptive algorithm is presented in Sec.~\ref{sec:adaptive}.
{ The gap dependent regret analysis and $\log T$-type regret bounds are provided in Sec.~\ref{sec:gap}.}
Numerical illustration are provided in Sec. \ref{sec:numerical} with the conclusion in Sec. \ref{sec:conclusion}.
The proof the lower bound result is provided in the appendix.
Proofs of some technical lemmas are relegated to the supplementary material.


\section{Related works}
\label{sec:rel}
Static assortment optimization with known choice behavior has been an active research area since the seminal works by \cite{Ryzin1999} and \cite{Mahajan2001}.  Motivated by fast-fashion retailing, dynamic assortment optimization, which adaptively learns unknown customers' choice behavior, has received increasing attention in the context of data-driven revenue management. The work by \cite{Caro2007} first studied dynamic assortment optimization problem under the assumption that demands for different products are independent. Recent works by \cite{Rusmevichientong2010, Saure2013, Agrawal17MNLBandit, Agrawal17Thompson, Chen:18tight, Chen:18near} incorporated  MNL models into dynamic assortment optimization and formulated the problem as an online regret minimization problem. In particular, for the standard MNL model, \cite{Agrawal17MNLBandit} and \cite{Agrawal17Thompson} developed UCB and Thompson sampling based approaches for online assortment optimization. Moreover, some recent work \citep{Wang:17:person, Chen:18context, MinOh:19} study dynamic assortment optimization based on contextual MNL models, where the utility takes the form of an inner product between a feature vector and the coefficients. 
The present work focuses on the standard non-contextual MNL model, but a natural direction for future work is to extend our results to the contextual setting. 
  
All works outlined above assume an underlying MNL choice model is correctly specified. 
However, model mis-specification is common in practice, and robust statistics, one of the most important branches in statistics, is a natural tool to address such mis-specification.
The $\varepsilon$-contamination model, which was proposed  by P. J. Huber \citep{Huber:64}, is perhaps the most widely used robust model and has recently attracted much attention from the machine learning community (see, e.g., \cite{Chen:16:Huber, Diakonikolas:17,Diakonikolas:18} and reference therein). 
Despite this attention, online learning in the $\varepsilon$-contamination model or its generalizations is relatively unexplored. 
In the online setting, \cite{Esfandiari:18allocation} studied online allocation under a mixing adversarial and stochastic model but the setting does not require any learning component.
For online learning, the recent works of \cite{Lykouris:18,gupta2019better} studied the contaminated stochastic multi-armed bandit (MAB), but, due to the complex structure of discrete choice models, these results do not directly apply to our setting. 
Indeed, a straightforward analogy between assortment optimization and MAB is to treat each feasible assortment as an arm, but directly using this mapping will result in a large regret due to the exponentially many possible assortments.  

In learning and decision-making settings, a few recent work investigate the impact of model mis-specification in revenue management, e.g., \cite{Cooper:06} for capacity booking problems and \cite{Besbes:15linear} for dynamic pricing.
In particular, \cite{Besbes:15linear} show that a class of pricing policies based on linear demand functions perform well even when the underlying demand is not linear.
\cite{Cooper:06} also identified some cases where simple decisions are optimal under mis-specification.
However, our setting is quite different, as the widely used UCB and Thompson sampling policies are not robust under our model. 
On the other hand, our new active-elimination policy is robust to model mis-specification and additional achieves near-optimal regret when the model is well-specified. 

{
Finally, the successive-elimination or active elimination strategies have been extensively studied in the (stochastic) multi-armed bandit literature.
Interested readers can refer to the works of \cite{auer2002using,auer2010ucb,even2006action} for details.
}


\section{Problem formulation}
\label{sec:formulation}


There are $N$ items, each associated with a known revenue parameter $r_i\in[0,1]$ and an unknown utility parameter $v_i\in[0,1]$.
At each time $t$ a customer arrives, for a total of $T$ time periods. 
The retailer then provides an \emph{assortment} $S_t\subseteq[N]$ to the customer, subject to a capacity constraint $|S_t|\leq K$.
The customer then chooses \emph{at most} one item $i_t\in S_t$ to purchase, upon which the retailer collects a revenue of $r_{i_t}$.
If the customer chooses to purchase nothing (denoted by $i_t=0$), then the retailer collects no revenue.

At each time $t$, the arriving customer is assumed to be one of the following two types:
\begin{enumerate}
\item A \textbf{typical} customer makes purchases $i_t\in S_t\cup\{0\}$ according to a multinomial-logit (MNL) choice model
\begin{equation}
\Pr[i_t=i|S_t] = \frac{v_i}{v_0 + \sum_{j\in S_t}v_j}, \;\;\;\;\;\; v_0 = 1.
\label{eq:mnl}
\end{equation}
We assume that $v_i\in[0,1]$;
\item An \textbf{outlier} customer makes purchases $i_t\in S_t\cup\{0\}$ according to an arbitrary unknown distribution $Q_t$ (marginalized on $S_t\cup\{0\}$). $Q_t$ can potentially change with $t$.
\end{enumerate}

We note that the MNL model in Eq.~(\ref{eq:mnl}) together with the constraint that $v_i\in[0,1]$ implies that ``no purchase'' is 
the most probable (or equally probable) outcome for
a \emph{typical} customer. This assumption has been made in operations literature, see, e.g., \cite{Agrawal17Thompson}.
{\blue 
Such an assumption that $v_i\leq 1$ for all $i$ is,
however, only for the ease of presentation, and the assumption can be easily relaxed to $v_i\leq C_v$
	for some known constant upper bound $C_v>0$.
	With the relaxed boundedness condition, one can enlarge the constructed confidence intervals $\hat\Delta_{\bar\varepsilon}(\tau+1)$ (see the definition in Algorithm \ref{alg:active-elimination})
	by multiplying a $C_v$ factor, and the other parts of our analysis/algorithms remain the same.
}

We consider the following $\varepsilon$-contamination model:
\begin{enumerate}
\item[(A1)] (Bounded adversaries). The number of outlier customers throughout $T$ time periods does not exceed $\varepsilon T$, where $\varepsilon\in[0,1)$ is a problem parameter;
\item[(A2)] (Adaptive adversaries). The choice model $Q_t$ for an outlier customer at time $t$ can be \emph{adversarially} and \emph{adaptively}
chosen, based on the previous customers, offered assortments,  and past purchasing activity.
\end{enumerate}

A rigorous mathematical formulation is as follows:
For any time period $t=1,2,\cdots,T$, let $\phi_t\in\{0,1\}$ be the indicator variable of whether customer at time $t$ is an outlier ($\phi_t=1$ if customer $t$
is an outlier and $0$ otherwise), $S_t\subseteq[N]$ be the assortment provided at time $t$, $i_t\in S_t\cup\{0\}$ be the purchasing activity 
of the customer.
The protocol is formally defined as follows:
\begin{definition}[Definition of protocol]
We define the following:
\begin{enumerate}
\item An \emph{adaptive adversary} consists of $T$ arbitrary measurable functions $\mathfrak A_1,\cdots,\mathfrak A_T$,
where $\mathfrak A_t:\{\phi_\tau,Q_\tau,S_\tau,i_\tau\}_{\tau \leq t-1}\mapsto (\phi_t,Q_t)$ produces the type of the customer (typical or outlier) $\phi_t$
and the outlier distribution $Q_t$ at time period $t$, from the filtration $\mathcal F_{t-1}=\{\phi_\tau,Q_\tau,S_\tau,i_\tau\}_{\tau\leq t-1}$;
\item An \emph{admissible policy} consists of $T$ random functions $\mathfrak P_1,\cdots,\mathfrak P_T$,
where $\mathfrak P_t: \{S_\tau,i_\tau\}_{\tau \leq t-1}\mapsto S_t$ produces a randomized assortment $S_t\subseteq[N]$, $|S_t|\leq K$ at time period $t$,
from the filtration $\mathcal G_{t-1}=\{S_\tau,i_\tau\}_{\tau\leq t-1}$; 
\item If $\phi_t=0$ then $i_t$ is realized according to model (\ref{eq:mnl}) conditioned on $S_t$;
otherwise if $\phi_t=1$ then $i_t$ is realized according to model $Q_t$. 
\end{enumerate}
\end{definition}


The objective of the retailer is to develop an admissible dynamic assortment optimization strategy that is competitive 
with a certain ``benchmark'' assortment. 
{
Unlike the classical setting, the definition of regret is a bit more complicated due to the
presence of both typical and adversarial customers. 
To shed light on the subtle differences between different benchmark assortments, in this paper we consider two different types
of cumulative regret, as introduced below.
To simplify notations we use $P_t$ to denote the customer's choice model at time $t$.
More specifically, $P_t$ is the ``typical'' model in Eq.~(\ref{eq:mnl}) (denoted as $P_t=\{v\}$) if a typical customer arrives at time $t$, 
and $P_t=Q_t$ if an outlier customer arrives at time $t$.
We use $R(S;P)$ to denote the expected revenue collected by offering assortment $S$ if the customer's choice model is modeled by $P$.
\label{page:regret}
\begin{enumerate}
\item The \emph{Typically-Optimal-Typically-Evaluated} ($\tote$) regret is defined as
\begin{equation}
\mathrm{Regret}^\tote(T) := \mathbb E\left[\sum_{t=1}^T R(S^*;\{v\})-R(S_t;\{v\})\right],
\label{eq:regret}
\end{equation}
where $S^* = \arg\max_{S\subseteq[N],|S|\leq K}R(S;\{v\})$ is the optimal assortment \emph{for typical customers};


\item The \emph{Best-In-Hindsight} ($\bih$) regret is defined as
\begin{equation}
\mathrm{Regret}^\bih(T) := \max_{S\subseteq[N],|S|\leq K}\mathbb E\left[\sum_{t=1}^T R(S;P_t)-R(S_t;P_t)\right].
\label{eq:regret-bih}
\end{equation}
\end{enumerate}

The $\tote$-regret uses the optimal assortment for typical customers $S^*$ as the benchmark.
Furthermore,  the $\tote$-regret  is always measured in
the difference of expected revenue on typical customers, regardless of whether a typical or an outlier customer is present at time $t$.
On the other hand, the $\bih$-regret measures the performance differences on the actual choice model $P_t$
of the incoming customers. In other words,  it compares the performance of the dynamic assortment planning algorithm with the optimal assortment on both typical and outlier customers.
The $\bih$-regret also coincides with the ``best stationary benchmark'' regret considered in most fully adversarial multi-armed bandit problems.

There is an important relationship between these two definitions of regret, as characterized in the following statement.
\begin{fact}
$\mathrm{Regret}^\bih(T) \leq \mathrm{Regret}^\tote(T) + \varepsilon T$.
\label{fact:regret-type}
\end{fact}
\begin{proof}
Let $S^*$ be the optimal assortment for typical customers and $\tilde S$ be the assortment
attaining the maximum in the definition of $\mathrm{Regret}^\bih(T)$.
Note that during time periods $t$ that $P_t=\{v\}$, the $R(\tilde S;\{v\})-R(S_t;\{v\}) \leq R(S^*;\{v\}-R(S_t;\{v\})$.
During time periods $t$ that $P_t=Q_t$, we have $|(R(\tilde S;Q_t)-R(S_t;Q_t)) - (R(S^*;Q_t)-R(S_t;Q_t))| \leq 1$,
because the expected revenue of any assortment under any choice model is at most one by normalization.
Since there are $\varepsilon T$ outlier time periods, we have that $\mathrm{Regret}^\bih(T) \leq \mathrm{Regret}^\tote(T) + \varepsilon T$.
\end{proof}

Fact \ref{fact:regret-type} shows that 
the difference between the $\tote$-regret and the $\bih$-regret is at most $\varepsilon T$.
Therefore, we shall focus solely on the $\tote$-regret in terms of the \emph{upper bound}, which always exhibits an $\varepsilon T$ additive term in the bounds.
Such an upper bound implies the same regret bound for $\mathrm{Regret}^\bih(T)$, up to a term of $\varepsilon T$.
For the lower bound, we consider the $\bih$-regret which is standard in the literature.
}

{
}

\section{An active-elimination policy}
\label{sec:act_policy}

To motivate our policy, we first briefly explain why the popular Upper-confidence-bounds (UCB) and Thompson sampling fail in the presence of outlier customers.
These algorithms are designed for the uncontaminated setting where $\varepsilon = 0$, so the confidence bounds (in UCB policies) and posterior updates (in Thompson sampling policies) are designed under the assumption that \emph{all} customers follow the same MNL model. 
Unfortunately, in the presence of outlier customers the confidence intervals are too narrow and the posterior updates are too aggressive. 
With these update strategies, a small number of outlier customers preferring items unpopular to typical customers could ``swing'' the algorithms' parameter estimates, which can lead to the belief that these unpopular items are actually popular.
This subsequently leads to poor exploration of the popular items, which eventually hurts performance.
As a numerical demonstration, we construct a concrete setting in Sec.~\ref{sec:numerical} where the performance of UCB and Thompson sampling policies degrades considerably in the presence of outlier customers.

We propose an \emph{active-elimination} policy for dynamic assortment optimization in the presence of outlier customers.
A pseudo-code description is given in Algorithm \ref{alg:active-elimination}.
While Algorithm \ref{alg:active-elimination} requires the knowledge of $\varepsilon$ (or an upper bound $\bar{\varepsilon}$, see Theorem \ref{thm:upper-bound})
as input, we emphasize that such requirement can be completely removed by designing more complex policies,
as we will show in Sec.~\ref{sec:adaptive}.
To highlight our main idea, we state Algorithm \ref{alg:active-elimination} upfront as the prior knowledge of $\varepsilon$ simplifies
both the algorithm and its analysis.

\begin{algorithm}[t]
\caption{An active-elimination algorithm for robust dynamic assortment optimization.}
\begin{algorithmic}[1]
\State \textbf{Input}: time horizon $T$, outlier proportion $\overline\varepsilon$, revenue parameters $\{r_i\}$, capacity constraint $K$.
\State \textbf{Output}: a sequence of assortments $\{S_t\}_{t=1}^T$ attaining good regret.
\State Set $\hat v^{(0)}\equiv 1$, $\hat\Delta_{\overline\varepsilon}(0)=1$, $\mathcal A^{(0)} = [N]$,
$T_0 = 128(K+1)^2N\ln T$;
\For{$\tau=0,1,2,\cdots$}
	\State \textsuperscript{*}Compute $S_\tau^{(i)} = \arg\max_{S\subseteq\mathcal A^{(\tau)}, |S|\leq K,i\in S} R(S;\hat v^{(\tau)})$ for every $i\in\mathcal A^{(\tau)}$; \label{line:S_tau} 
	\State Compute $\gamma^{(\tau)} = \max_{i\in \mathcal A^{(\tau)}}R(S_\tau^{(i)};\hat v^{(\tau)})$; \label{line:gamma_tau}
	\State Update $\mathcal A^{(\tau+1)} = \{i\in\mathcal A^{(\tau)}: R(S_\tau^{(i)};\hat v^{(\tau)})+ 2\hat\Delta_{\overline\varepsilon}(\tau) \geq \gamma^{(\tau)}\}$;\label{line:active-elimination}
	\State Set $n_i=0$ and $n_0(i)=0$ for all $i\in\mathcal A^{(\tau+1)}$; set $T_\tau = 2^\tau T_0$;
	\For{the next $T_\tau$ time periods}
		\State Sample $i\in\mathcal A^{(\tau+1)}$ uniformly at random;\label{line:randomization}
		\State Provide the assortment $S_\tau^{(i)}$ to the incoming customer and observe purchase $i_t$;
		\State Update $n_i\gets n_i + \vct 1\{i_t=i\}$ and $n_0(i)\gets n_0(i)+\vct 1\{i_t=0\}$; \label{line:n_0}
	\EndFor
	\State Update estimates $\hat v_i^{(\tau+1)}=\max\{1,n_i/n_0(i)\}$ for every $i\in\mathcal A^{(\tau+1)}$; 
	\State Define $\overline\varepsilon_\tau = \min\{1,\overline\varepsilon T/T_\tau\}$, $N_\tau=|\mathcal A^{(\tau+1)}|$ and compute error upper bound as
\begin{equation*}
\textstyle
\hat\Delta_{\overline\varepsilon}(\tau+1)=
\textstyle
\left\{\begin{array}{ll}
1,& T_\tau<\frac{\overline\varepsilon T}{4(K+1)};\\
16K(K+1)\left(\frac{\overline\varepsilon_\tau}{2} + \sqrt{\frac{\overline\varepsilon_\tau N_\tau\ln T}{T_\tau}} + \frac{2N_\tau\ln T}{3T_\tau}\right) + 16\sqrt{\frac{KN_\tau\ln T}{T_\tau}},& \text{otherwise};\\\end{array}\right.
\end{equation*}
\EndFor
\State Remarks:\\
{\footnotesize
\textsuperscript{*} For any set of $\{\hat v\}$, $R(S;\hat v) = (\sum_{i\in S}r_i\hat v_i)/(1+\sum_{i\in S}\hat v_i)$; the optimization can be computed efficiently.
See Sec.~\ref{subsec:computation} for details.
}
\end{algorithmic}
\label{alg:active-elimination}
\end{algorithm}
 
At a high level, Algorithm \ref{alg:active-elimination} operates in \emph{epochs} $\tau=0,1,\cdots$ with geometrically increasing lengths,
and only performs item estimation or assortment updates between epochs.
At any time $t$, the algorithm maintains an active set of items $\mathcal A\subseteq[N]$ consisting of all items that could potentially form a ``good'' assortment,
and estimates of parameters $\{\hat v_i\}$ for all active items $i$ in $\mathcal A$.
For each time period $t$ in a single epoch $\tau$, a \emph{random} item $i$ is sampled from the current active item set and a ``near-optimal'' assortment is built, which must contain the target item $i$.
Once an epoch $\tau$ ends, parameter estimates of $\{\hat v_i\}$ are updated and the active set $\mathcal A$ is shrunk based on the updated estimates to exclude sub-optimal items. {We will ensure that with high probability, the optimal assortment $S^*$ is always a subset of active sets for all epochs (see Lemma \ref{lem:feasible}).}

We now detail all notation used in Algorithm \ref{alg:active-elimination}:
\begin{itemize}
\item[-] $\tau\in\mathbb N$: the indices of \emph{epochs} whose lengths increase geometrically ($T_\tau = 2^\tau T_0$);
\item[-] $\hat v^{(\tau)}\in [0,1]^N$: the estimates of preference parameters (of typical customers) at epoch $\tau$;
\item[-] $\mathcal A^{(\tau+1)}\subseteq[N]$: the subset of active items, which are to be explored uniformly at random in epoch $\tau$;
\item[-] $\gamma^{(\tau)}\in[0,1]$ (see step \ref{line:gamma_tau}): the estimated expected revenue of the optimal assortment calculated based on the active item subset $\mathcal A^{(\tau+1)}$ and current preference estimates $\hat v^{(\tau)}$;
\item[-] $S_\tau^{(i)}\subseteq[N]$ (see step \ref{line:S_tau}): an optimal assortment computed based on $\mathcal A^{(\tau+1)}$ and $\hat v^{(\tau)}$, which \emph{must include} the specific item $i$;
this assortment is used to explore and estimate the the utility parameter $v_i$ of item $i$;
\item[-] $n_i, n_0(i)\in\mathbb N$ (see step \ref{line:n_0}): counters used in the estimate of $v_i$; note that for any supplied assortment $S_\tau^{(i)}$, we only record the number of times a customer purchases item $i$ (accumulated by $n_i$),
and the number of times a customer makes no purchases (accumulated by $n_0(i)$); other purchasing activities (e.g., purchases of an item $\ell\in S_\tau^{(i)}$ other than $i$) will not be recorded;
\item[-] $\hat\Delta_{\bar\varepsilon}(\tau+1)\in[0,1]$: length of confidence intervals used to eliminate items from $\mathcal A^{(\tau+1)}$; its length depends on both the epoch index $\tau$
and the prior knowledge of the outlier proportion $\bar\varepsilon$;
\end{itemize}

In the rest of the section, we first give a brief description of how to compute $\hat S_\tau^{(i)}$  in Line \ref{line:S_tau} efficiently.
Then we detail the regret upper bound of Algorithm \ref{alg:active-elimination} and provide the the proof.

\subsection{Solving the optimization problem}\label{subsec:computation}

\newcommand{\MID}{\mathrm{mid}}
\begin{algorithm}[!t]
	\caption{Assortment optimization with additional constraints}
	\begin{algorithmic}[1]
		\State \textbf{Input}: revenue parameters $\{r_i\}_{i=1}^n$, estimated preference parameters $\{\hat v_i\}_{i=1}^n$, must-have item $i$, capacity constraint $K$,
		stopping accuracy $\delta$;
		\State\textbf{Output}: assortment $\hat S$, $|\hat S|\leq K$, $i\in\hat S$ that maximizes $R(\hat S;\hat v)$.
		\State Initialization: $\alpha_\ell=0$ and $\alpha_u = 1$; $\hat S=\emptyset$; 
		\While{$\alpha_u-\alpha_\ell \geq \delta$}
		\State $\alpha_{\MID}\gets (\alpha_\ell + \alpha_u) / 2$;
		\State For each $j\neq i$, sort $\psi_j := (r_j-\alpha_\MID)\hat v_j$ in descending order, and let $\Psi:=\{j\neq i: \psi_j\geq 0\}$ be the subset consisting
		of all items other than $i$ with non-negative $\psi_j$;
		\State Compute $t := \psi_{i} + \text{ the $(K-1)$ $\psi_j$ in $\Psi$ with the largest values}$;
		\State If $t\geq\alpha_\MID$ then set $\hat S=\{i\}\cup\{\text{the $(K-1)$ items in $\Psi$ with the largest $\psi_j$ values}\}$ and $\alpha_\ell\gets\alpha_{\MID}$;
		else set $\alpha_u\gets \alpha_\MID$.
		\EndWhile
	\end{algorithmic}
	\label{alg:optimization}
\end{algorithm}

The implementation of most steps of Algorithm \ref{alg:active-elimination} is straightforward,
except for the computation of the assortments $S_\tau^{(i)}$, which require futher algorithmic development. 
This computation can be formulated as the following combinatorial optimization problem:
\begin{equation}
\max_{|S|\leq K,i\in S} R(S;\hat v) = \max_{|S|\leq K,i\in S}\frac{\sum_{j\in S}r_j\hat v_j}{1+\sum_{j\in S}\hat v_j},
\label{eq:si-opt}
\end{equation}
for a specific $i\in[N]$.
This optimization problem is similar to the classical capacity-constrained assortment optimization  (see, e.g., \cite{Rusmevichientong2010}),
but the additional constraint $i \in S$ in~\eqref{eq:si-opt} yields a subtle difference.
For the purpose of completeness, we provide an efficient optimization method with binary search for solving Eq.~(\ref{eq:si-opt}).
Pseudo-code is provided in Algorithm~\ref{alg:optimization}.

For any $\alpha\in(0,1]$, we want to check whether there exists $S\subseteq[N]$, $|S|\leq K$, $i\in S$ such that $R(S;\hat v)\geq \alpha$,
or equivalently $\sum_{j\in S}r_j\hat v_j\geq \alpha+ \alpha\sum_{j\in S}\hat v_j$.
Re-organizing the terms, we only need to check whether there exists $|S|\leq K$, $i\in S$ such that $\sum_{j\in S}(r_j-\alpha)\hat v_j\geq\alpha$.
Because $i\in S$ must hold, we only need to check whether there exists $S'\subseteq[N]\backslash\{i\}$, $|S'|\leq K-1$ such that
\begin{equation}
(r_i-\alpha)\hat v_i + \sum_{j\in S'}(r_j-\alpha)\hat v_j \geq\alpha.
\label{eq:opt-includei}
\end{equation}
This can be accomplished by including all $j\in [N]\backslash\{i\}$ with the largest $(K-1)$ \emph{positive} values of $(r_j-\alpha)\hat v_j$ into the set of $S'$
and check whether Eq.~(\ref{eq:opt-includei}).
If Eq.~(\ref{eq:opt-includei}) holds, the current revenue value of $\alpha$ can be obtained and otherwise the current value of $\alpha$ cannot be obtained.
We then solve the optimization problem by a standard binary search on $\alpha$. {We also note that $\gamma^{(\tau)}$ in Line \ref{line:gamma_tau} is a standard  static capacitated assortment optimization, which can be solved efficiently (see \cite{Rusmevichientong2010}).}

\subsection{Regret analysis}
\label{sec:regret}
The following theorem is our main regret upper bound result for Algorithm \ref{alg:active-elimination}.
\begin{theorem}
\label{thm:upper-bound}
Suppose $\overline\varepsilon\geq\varepsilon$ and $N\leq T$.
Then there exists a universal constant $C_0<\infty$ such that, for sufficiently large $T$,
 the $\tote$-regret of Algorithm \ref{alg:active-elimination} is upper bounded by 
$$
C_0\times \left(\overline\varepsilon K^2T\log T + (K^2\sqrt{\overline\varepsilon}+\sqrt{K})\sqrt{NT\log^3 T} + K^2 N\log^2 T\right).
$$
Furthermore, if $\overline\varepsilon\lesssim 1/K^3$ holds then the regret upper bound can be simplified to
\begin{equation}\label{eq:TOTE-upper}
O\left(\overline\varepsilon K^2 T\log T + \sqrt{KNT\log^3 T}\right).
\end{equation}
\end{theorem}
{
Combined with Fact \ref{fact:regret-type},  we know that  Eq.~\eqref{eq:TOTE-upper} also serves as an upper bound for the $\bih$-regret. 
} 

To complement Theorem \ref{thm:upper-bound},
we state the following proposition establishing some lower bounds for the different types of regret considered in this paper.
{
\begin{proposition}
Let $c_0>0$ be a universal constant and $\pi$ be any admissible policy. Suppose also $K < N/4$.
\begin{enumerate}
\item The  $\bih$-regret of $\pi$ on worst-case problem instances are at least $c_0\times \sqrt{NT}$;
\item For $0\leq\varepsilon <1$ suppose there are $\lfloor\varepsilon T\rfloor$ outlier customers.
Then the $\tote$-regret of $\pi$ on worst-case problem instances is lower bounded by at least $c_0\times (\varepsilon T+\sqrt{NT})$.
\end{enumerate}
\label{prop:lowerbound}
\end{proposition}

The first property of Proposition \ref{prop:lowerbound} is proved by simply setting $\varepsilon=0$ and using existing lower bound
results for dynamic assortment planning with no outlier customers (see, e.g., \cite{Chen:18tight}).
The proof of the second property is achieved by considering the two terms $\varepsilon T$ and $\sqrt{NT}$ separately.
The complete proof of Proposition \ref{prop:lowerbound} is given in the supplementary material.

The claims in Proposition \ref{prop:lowerbound} leads to a challenging open problem on the $\bih$-regret upper bound 
when $\varepsilon \gtrsim \sqrt{N/T}$, at which time the $\varepsilon T$ term would dominate the $\sqrt{NT}$ term (see Eq.~\eqref{eq:TOTE-upper}).
In such cases, we conjecture that the optimal regret upper bounds would be $\sqrt{NT}$, implying that our current result in Theorem \ref{thm:upper-bound}
is sub-optimal when $\varepsilon$ is very large.
The question of achieving $\widetilde O(\sqrt{NT})$ regret upper bound for \emph{all} $\varepsilon$ levels requires fully adversarial bandit algorithms
for dynamic assortment optimization, which is very challenging and an open question as far as we know.
}


 
An important special case of Theorem \ref{thm:upper-bound} is $\varepsilon=\bar\varepsilon=0$,
which reduces to the well-studied dynamic assortment optimization problem  without outlier customers.
For such settings, \cite{Agrawal17Thompson,Agrawal17MNLBandit} give algorithms with a regret upper bound of $\tilde O(\sqrt{NT})$,
which matches the lower bound of $\Omega(\sqrt{NT})$ given in \citep{Chen:18tight} up to poly-logarithmic terms.
Comparing their results to Theorem \ref{thm:upper-bound}, we observe that our result at $\varepsilon=\bar\varepsilon=0$ matches the $\tilde O(\sqrt{NT})$ regret bound
except for an additional term of $O(\sqrt{K})$.
This $O(\sqrt{K})$ factor stems from our active elimination protocol and our technique for estimating the utility parameters, both of which are essential for handling outlier customers when $\varepsilon>0$.
We believe removing this factor is technically quite challenging, and leave it as an interesting open question. 
We also note that the capacity constraint $K$ is typically a very small constant in practice, and hence an additional $O(\sqrt{K})$ term is likely negligible. 

Our regret upper bound in Theorem \ref{thm:upper-bound} also yields meaningful guarantees when $\varepsilon$ is not zero.
For example, with $\varepsilon=O(T^{-1/4})$, meaning that $O(T^{3/4})$ out of $T$ customers are outliers, 
Theorem \ref{thm:upper-bound} provides an $O(K^2T^{3/4}\log T)$ regret upper bound.
This guarantee is non-trivial because it is sub-linear in $T$, although it is larger than the standard $\tilde O(\sqrt{NT})$ bound for the uncontaminated setting.
Thus, Theorem~\ref{thm:upper-bound} reveals the trade-off and impact of a small proportion of outlier customers on the performance of dynamic assortment optimization algorithms/systems.

\subsection{Proof sketch of Theorem \ref{thm:upper-bound}}
In this section we sketch the proof of Theorem \ref{thm:upper-bound}. 
Key lemmas and their implications are given, while the complete proofs of the presented lemmas
are deferred to the supplementary material accompanying this paper.

We first state a lemma that upper bounds the estimation error $|\hat v_i^{(\tau+1)}-v_i|$:
\begin{lemma}
Suppose $T_0\geq 128(K+1)^2N_\tau \ln T$ and $\min\{1,\varepsilon T/T_\tau\}\leq 1/4(K+2)$.
With probability $1-O(\tau_0 N/T^2)$ it holds for all $\tau$ satisfying $T_\tau\geq\max\{\overline\varepsilon,\varepsilon\} T/4(K+1)$ and $i\in\mathcal A^{(\tau+1)}$ that 
$|\hat v_i^{(\tau+1)}-v_i|\leq \Delta_\varepsilon^*(i,\tau+1)$, where
\begin{equation}
\Delta_\varepsilon^*(i,\tau+1) = 
8(K+1)\left(\frac{\varepsilon_\tau}{2} + \sqrt{\frac{\varepsilon_\tau N_\tau\ln T}{T_\tau}} + \frac{2N_\tau \ln T}{3T_\tau}\right) + 8\sqrt{\frac{(1+V_S)v_iN_\tau\ln T}{T_\tau}},
\label{eq:delta-star}
\end{equation}
where $\varepsilon_\tau$ is defined as $\varepsilon_\tau=\min\{1,\varepsilon T/T_\tau\}$, $N_\tau=|\mathcal A^{(\tau+1)}|$ and $V_S= \sum_{j\in S_\tau^{(i)}}v_j$.
\label{lem:ucb-lcb}
\end{lemma}

Lemma \ref{lem:ucb-lcb} shows that, with high probability, the estimation error between $\hat v_i^{(\tau+1)}$ and $v_i$, the true preference parameter
of item $i$ for typical customers, can be upper bounded by $\Delta_{\varepsilon}^*(i,\tau+1)$ which is a function of $K$, $\tau$, $T$, $\varepsilon$ and
$N_\tau=|\mathcal A^{(\tau+1)}|$.
It should be noted that the definition of $\Delta_{\varepsilon}^*(i,\tau+1)$ involves unknown quantities (mostly $V_S=\sum_{j\in S_\tau^{(i)}}v_j$)
and hence cannot be directly used in an algorithm.
The definition of $\hat\Delta_{\bar\varepsilon}(\tau+1)$ in Algorithm \ref{alg:active-elimination}, on the other hand, involves only 
known quantities and estimates.
In Corollary \ref{cor:rdiff}, we will establish the connection between $\Delta_{\varepsilon}^*(i,\tau+1)$ and $\hat\Delta_{\bar\varepsilon}(\tau+1)$.

Our next lemma derives how the estimated expected revenue $R(S;\hat v)$ deviates from the true value $R(S;v)$
by using upper bounds on the estimation errors between $\hat v$ and $v$:
\begin{lemma}
\label{lem:revenue-error}
For any $S\subseteq[N]$, $|S|\leq K$ and $\{\hat v_i\}$, it holds that
$$
|R(S;\hat v)-R(S;v)|\leq\frac{2\sum_{i\in S}|\hat v_i-v_i|}{1+\sum_{i\in S}v_i}.
$$
\end{lemma}
The proof uses only elementary algebra. 

Combining Lemmas \ref{lem:ucb-lcb} and \ref{lem:revenue-error},
we show that the $\hat\Delta_{\bar\varepsilon}(\tau)$ quantities defined in our algorithm
serve as valid upper bounds on the estimation error between $R(S;\hat v^{(\tau)})$ and $R(S;v)$:

\begin{corollary}
\label{cor:rdiff}
For every $\tau$ and $|S|\leq K$, $S\subseteq \mathcal A^{(\tau)}$, conditioned on the success events of Lemma \ref{lem:ucb-lcb} on epochs up to $\tau$, it holds that
$|R(S;\hat v^{(\tau)})-R(S;v)|\leq \hat\Delta_{\varepsilon}(\tau)\leq\hat\Delta_{\max\{\varepsilon,\bar\varepsilon\}}(\tau)$, where $\hat\Delta$
is defined in Algorithm \ref{alg:active-elimination}. 
\end{corollary}

Our next lemma is an important structural lemma which states that,
with high probability, any item in the optimal assortment $S^*$ is never excluded from active item sets $\mathcal A^{(\tau+1)}$ for all epochs $\tau$.
\begin{lemma}
\label{lem:feasible}
If $\overline\varepsilon\geq\varepsilon$ then with probability $1-O(\tau_0 N/T^2)$ it holds that $S^*\subseteq\mathcal A^{(\tau)}$ for all $\tau$.
\end{lemma}

This structural lemma yields two important consequences: 
first, since ``good'' items remain within the active item subsets $\mathcal A^{(\tau+1)}$, each of the assortments $S_\tau^{(i)}$ computed at step \ref{line:S_tau} of Algorithm \ref{alg:active-elimination}
will have relatively high expected revenue. 
Second, the fact that $S^*\subseteq\mathcal A^{(\tau+1)}$ implies that the optimistic estimates $\gamma^{(\tau)}$ 
will always be based on the expected revenue of the actual optimal assortment $R(S^*;v)$.
This justifies the elimination step \ref{line:active-elimination} in which we discard all items whose best assortment has significantly lower revenue than $\gamma^{(\tau)}$. 

The proof of Lemma \ref{lem:feasible} is based on an inductive argument, which shows that if $S^*$ belongs to $\mathcal A^{(\tau)}$
at the beginning of every epoch $\tau$, then any item in $S^*$ will not be removed (with high probability) by step \ref{line:active-elimination}. 
The intuition for this is that the optimal assortment containing any $i\in S^*$ is $S^*$ itself, whose revenue cannot be to far away from $\gamma^{(\tau)}$ due to Lemmas \ref{lem:ucb-lcb} and \ref{lem:revenue-error}.
The complete proof of Lemma \ref{lem:feasible} is provided in the supplementary material.

Finally, our last technical lemma upper bounds the per-period regret incurred by Algorithm \ref{alg:active-elimination}.
\begin{lemma}
Suppose $S^*\subseteq\mathcal A^{(\tau)}$ holds for all $\tau$. Then with probability $1-O(\tau_0 N/T^2)$, for every $\tau\leq \tau_0$ and $i\in\mathcal A^{(\tau+1)}$, it holds that
$R(S^*;v)-R(S_\tau^{(i)};v) \leq 4\hat\Delta_{\overline\varepsilon}(\tau)$.
\label{lem:regret-single-iter}
\end{lemma}

Given the established technical lemmas, we are now ready to give the proof of Theorem \ref{thm:upper-bound}.

\begin{proof} 
Let $\tau^*$ be the smallest integer such that $T_{\tau^*}\geq \overline\varepsilon T/4(K+1)$.
For all epochs $\tau<\tau^*$, the induced cumulative regret can be upper bounded by 
\begin{equation}
\sum_{\tau<\tau^*}T_\tau \leq T_{\tau^*} \leq \overline\varepsilon T.
\label{eq:regret-exploration}
\end{equation}

In the rest of this proof we upper bound the regret incurred from epochs $\tau\geq\tau^*$.
By Lemma \ref{lem:regret-single-iter}, the regret incurred by a single time period in epoch $\tau$ is upper bounded by $4\hat\Delta_{\overline\varepsilon}(\tau)$ with high probability.
The total regret accumulated in epoch $\tau$ is then upper bounded by $4\hat\Delta_{\overline\varepsilon}(\tau) \times T_\tau$.
Hence, the regret accumulated on the entire $T$ time periods is upper bounded by 
\begin{align}
&\sum_{\tau=0}^{\tau_0}4\hat\Delta_{\overline\varepsilon}(\tau)  T_\tau \nonumber\\
&\lesssim \sum_{\tau=0}^{\tau_0}\left({K^2\overline\varepsilon_\tau} +K^2\sqrt{\frac{\overline\varepsilon_\tau|\mathcal A^{(\tau+1)}|\log T}{T_\tau}} +\frac{K^2|\mathcal A^{(\tau+1)}|\log T}{T_\tau} + \sqrt{\frac{K|\mathcal A^{(\tau+1)}|\log T}{T_\tau}}\right)\times T_\tau\nonumber\\
&\leq \sum_{\tau=0}^{\tau_0}\left(\frac{K^2\overline\varepsilon T}{T_\tau} +K^2\sqrt{\frac{\overline\varepsilon |\mathcal A^{(\tau+1)}|T\log T}{T_\tau^2}} +\frac{K^2|\mathcal A^{(\tau+1)}|\log T}{T_\tau} + \sqrt{\frac{K|\mathcal A^{(\tau+1)}|\log T}{T_\tau}}\right)\times T_\tau\label{eq:eps-epstau}\\
&\leq \tau_0 K^2\overline\varepsilon T+ K^2\sqrt{\overline\varepsilon T\log T}\bigg(\sum_{\tau\leq\tau_0} \sqrt{|\mathcal A^{(\tau+1)}|}\bigg) \nonumber\\
&\quad + \sqrt{K\log T}\bigg( \sum_{\tau\leq\tau_0}\sqrt{T_\tau|\mathcal A^{(\tau+1)}|}\bigg)+ K^2\log T\bigg(\sum_{\tau\leq\tau_0}|\mathcal A^{(\tau+1)}|\bigg)\nonumber\\
&\leq  \tau_0 K^2\overline\varepsilon T + \tau_0K^2\sqrt{\overline\varepsilon NT\log T} +\tau_0 K^2 N\log T + \sqrt{K\log T}\times\sqrt{\sum_{\tau\leq\tau_0}|\mathcal A^{(\tau+1)}|}\times \sqrt{\sum_{\tau\leq\tau_0}T_\tau}\label{eq:final-1}\\
&\leq K^2\overline\varepsilon T\log T + K^2\sqrt{\overline\varepsilon NT\log^3T} +  \sqrt{K\log T}\times \sqrt{\tau_0 N}\times \sqrt{T} + K^2N\log^2 T\nonumber\\
&\lesssim \overline\varepsilon K^2T\log T + (K^2\sqrt{\overline\varepsilon}+\sqrt{K})\sqrt{NT\log^3 T} + K^2 N\log^2 T.
\label{eq:final-2}
\end{align}
Here in Eq.~(\ref{eq:final-1}), we apply Cauchy-Schwartz inequality. 
The final inequality holds because $\tau_0=O(\log T)$. 
\end{proof}

\section{Adaptation to unknown outlier proportion $\varepsilon$}
\label{sec:adaptive}
In this section we describe a more complex algorithm for robust dynamic assortment optimization where the outlier proportion $\varepsilon$ is \emph{unknown}
a priori.
Inspired by the ``multi-layer active arm race'' for multi-armed bandits, due to \cite{Lykouris:18},
Algorithm \ref{alg:adaptive-epsilon} runs multiple ``threads'' of known-$\varepsilon$ algorithms on a geometric grid of $\varepsilon$ values in parallel, while carefully coordinating between the threads.
The pseudo-code of the proposed adaptive algorithm is given in Algorithm \ref{alg:adaptive-epsilon}. 

{We note that for two threads $j' <j$, we have $\hat\varepsilon_{j'}>\hat\varepsilon_{j}$, which implies that the confidence interval length 
$\hat\Delta_{\hat\varepsilon_{j'}}(\tau+1)$ is typically longer than $\hat\Delta_{\hat\varepsilon_{j}}(\tau+1)$. 
Therefore, the thread $j'$ is less aggressive than the thread $j$ in terms of eliminating items, i.e., an item eliminated by thread $j$ may remain active in thread $j'$.}
More detailed explanations of key steps in Algorithm \ref{alg:adaptive-epsilon} are summarized below: 

\begin{algorithm}[t]
\caption{Dynamic assortment optimization robust to unknown outlier proportion $\varepsilon$.}
\begin{algorithmic}[1]
\State \textbf{Input}: lower bound on outlier proportion $\underline\varepsilon=2^{-J}$, $J=\lfloor\log_2\sqrt{N/T}\rfloor + 1$;
\State \textbf{Output}: a sequence of assortments $\{S_t\}_{t}$ attaining good regret for any $\varepsilon$;
\State Construct a grid of outlier proportion values $\{\hat\varepsilon_j\}_{j=0}^{J-1}$ where $\hat\varepsilon_j=2^{-j}$; 
\State Construct $J$ threads $j<J$, each with $\hat\varepsilon_j$ outlier proportion;
\State For each $i\in[N]$ and $j< J$, set $\hat v^{(0),j}\equiv 1$, $\hat\Delta_{\hat\varepsilon_j}(0)=1$, $\mathcal A_j^{(0)}=[N]$, $T_0=64(K+1)^2\ln T$;
\For{$\tau=0,1,2\cdots$}
	\For{$j=0,1,\cdots,J-1$}
		\State If $j>0$ then update $\mathcal A_j^{(\tau)} = \mathcal A_j^{(\tau)}\cap \mathcal A_{j-1}^{(\tau+1)}$;\label{step:hierarchical-elim}
		\State \textsuperscript{*}Compute $\gamma_j^{(\tau)}$ and $S_{\tau,j}^{(i)}$ for each $i\in\mathcal A_j^{(\tau)}$ and update $\mathcal A_j^{(\tau+1)}$;
	\EndFor
	\For{the next $T_\tau=2^\tau T_0$ time periods}
		\State Sample thread $j< J$ with probability $\wp_j:=2^{-(J-j)}/(1-2^{-J})$;
		\State Sample item $i\in \mathcal A_j^{(\tau+1)}$ uniformly at random;
		\If{\textsuperscript{$\dagger$}there exists $\hat\varepsilon_k>\hat\varepsilon_j$ such that $R(\hat S_{\tau,j}^{(i)};\hat v^{(\tau),k})< \gamma_k^{(\tau)}-7\hat\Delta_{\hat\varepsilon_k}(\tau)$}\label{step:check}
			\State Re-start Algorithm \ref{alg:adaptive-epsilon} with $J\gets J-1$;\label{step:stopping}
		\EndIf
		\State Provide assortment $S_{\tau,j}^{(i)}$ to the incoming customer and observes purchase $i_t$;
		\State Update $n_i^j \gets n_i^j +\vct 1\{i_t=i\}$ and $n_0^j(i)\gets n_0^j(i)+\vct 1\{i_t=0\}$;
	\EndFor
	\State Update estimates $\hat v_{i}^{(\tau+1),j}=\max\{1, n_i^j/n_0^j(i)\}$ for all $j\leq J$ and $i\in\mathcal A_j^{(\tau+1)}$;
	\State For every $j\leq J$, compute $\hat\Delta_{\hat\varepsilon_j}(\tau+1)$ with $T,T_\tau$ replaced by $T_j:=\wp_j T$ and $T_{\tau,j}:=\wp_j T_\tau$;
\EndFor
\State 
{\footnotesize\textsuperscript{*} Using the procedure outlined in Algorithm \ref{alg:optimization}.}\\
{\footnotesize \textsuperscript{$\dagger$} $\hat v^{(\tau),k}$ and $\gamma_k^{(\tau)}$ are estimates of $v$ and computed $\gamma^{(\tau)}$ values maintained in thread $k$.}
\end{algorithmic}
\label{alg:adaptive-epsilon}
\end{algorithm}

\begin{enumerate}
\item \textbf{Independence of threads}:  different threads $j<J$, which correspond to different hypothetical values of $\varepsilon$ (denoted as $\hat\varepsilon_j$), 
are largely independent from each other, maintaining their own parameter estimates $\hat v^{(\tau),j}$, active item set $\mathcal A_j^{(\tau+1)}$ and confidence intervals $\hat\Delta_{\hat\varepsilon_j}(\tau+1)$.
Coordination among threads only appear in two steps in Algorithm \ref{alg:adaptive-epsilon}: Step \ref{step:hierarchical-elim}, which maintains a hierarchical ``nested'' structure
of the active item sets $\mathcal A_j^{(\tau+1)}$ among the threads,
and Step \ref{step:stopping}, which provides update rules for $J\gets J-1$ by comparing the obtained optimistic assortment among different threads.
Further details are given in subsequent bullets.

\item\textbf{Heterogeneous sampling of different threads}: at each time period $t$ when a potential customer arrives, a \emph{random} thread $j<J$ is selected to provide assortments.
The random thread, however, is not selected uniformly at random but according to a specifically designed distribution, with the probability of selecting thread $j$ equals $\wp_j=2^{-(J-j)}/(1-2^{-J})$.
Intuitively, such a sampling distribution ``favors'' the more aggressive threads with smaller hypothetical $\hat\varepsilon_j$ values.

This sampling scheme is motivated by the fact that threads with larger $\hat\varepsilon_j$ values typically incur large regret,
because their elimination rules are conservative, so many sub-optimal items $i$ remain active for many rounds. 
The probability of choosing these threads with large $\hat\varepsilon_j$ values should be small to ensure low regret of the overall policy.

At the same time, threads corresponding to smaller $\hat\varepsilon_j$ values might also incur large regret, as their overly aggressive elimination rule might remove the optimal assortment $S^*$ from consideration.
To avoid large regret from these threads, Step~\ref{step:stopping} coordinates amongst all of the threads and checks for inconsistencies, as we describe in the next bullet.

\item\textbf{Coordination and interaction among threads}: as we mentioned in the first bullet, the coordination and interaction among different threads only happen in Steps \ref{step:hierarchical-elim} and \ref{step:stopping}
in Algorithm \ref{alg:adaptive-epsilon}. In this bullet we discuss these two steps in detail.

Step \ref{step:hierarchical-elim} aims at maintaining a ``nested'' structure among the active subsets $\mathcal A_j^{(\tau+1)}$, such that $\mathcal A_j^{(\tau+1)}\subseteq\mathcal A_{j'}^{(\tau+1)}$ for
any $j'\leq j$ at any epoch $\tau$.
We remark that such a nested structure should be expected even without this step, because thread $j'\leq j$ is less aggressive than thread $j$, in the sense that confidence intervals $\hat\Delta_{\hat\varepsilon_{j'}}(\tau+1)$
is typically longer than $\hat\Delta_{\hat\varepsilon_{j}}(\tau+1)$. Hence, one should expect that thread $j'$ has a larger active set.
Nevertheless, due to stochastic fluctuations such nested structures might be violated.
Therefore, we explicitly enforce a nesting structure at the start of every epoch $\tau$ via Step~\ref{step:hierarchical-elim}.

Step \ref{step:stopping} is a statistical test that tries to detect whether $\hat\varepsilon_{j}$ is small relative to the actual (unknown) outlier proportion $\varepsilon$.
This test crucially ensures that we do not continue to select an overly aggressive thread, which, as we have mentioned, may incur large regret due to eliminating the optimal assortment $S^*$. 
Step \ref{step:stopping} detects such events by evaluating the optimistic assortment $S_{\tau,j}^{(\cdot)}$ using the information from threads $j' < j$, which use less aggressive elimination rules.
In detail, we check if the optimistic assortment $S_{\tau,j}^{(\cdot)}$ is near optimal using the utility estimates and confidence intervals from thread $j'$.
If the check fails and we see that $S_{\tau,j}^{(\cdot)}$ is suboptimal, we know that thread $j$ has eliminated the optimal assortment $S^*$ from its active set $\mathcal A_j^{(\cdot)}$, 
which subsequently lead to the conclusion that $\hat\varepsilon_j$ is too small. {Then we terminate the current thread and restart the algorithm with $J\gets J-1$}. 
\end{enumerate}

{
We also remark on the time complexity of Algorithm \ref{alg:adaptive-epsilon}.
There are $O(\log (T/N))$ values on the $\varepsilon$-grid.
		At each time period $t$, a thread $\hat\varepsilon_j$ is chosen. Then at most $N$ combinatorial optimization problems are solved
		and each combinatorial optimization takes $O(NK\log T)$ time.
		Therefore, the total time complexity of the proposed algorithm is $O(NKT\log^2 T)$.
}

In the rest of this section we state our regret upper bound result for the adaptive Algorithm \ref{alg:adaptive-epsilon},
as well as a sketch of its proof.

\subsection{Regret analysis and proof sketch}

We establish the following regret upper bound for Algorithm \ref{alg:adaptive-epsilon}. {\blue We note that all the regret mentioned in this section is the $\tote$-regret.}

\begin{theorem}
\label{thm:adaptive-epsilon}
Suppose Algorithm \ref{alg:adaptive-epsilon} is run with an initial value of $J=\lfloor \log_2(\sqrt{N/T})\rfloor + 1$.
Then there exists a constant $C_1=\poly(K,\log (NT))$ such that, for any $\varepsilon\in[0,1/2]$ and sufficiently large $T$, the regret of Algorithm \ref{alg:adaptive-epsilon} is upper bounded by 
$$
C_1\times\big(\varepsilon T+\sqrt{NT}\big).
$$
\end{theorem}
\begin{remark}
In the statement of Theorem \ref{thm:adaptive-epsilon}, $C_1=\poly(K,\log(NT))$ means $C_1=(K\log (NT))^{c}$ for some universal constant $c<\infty$. For notational simplicity we did not work out the exact constant $c$ in the expression of $C_1$.
\end{remark}

The complete proof of Theorem \ref{thm:adaptive-epsilon} as well as the proofs of technical lemmas are relegated to the supplementary material. Here we sketch the key steps in the proof.
The first step is the following lemma, which shows that for threads with $\hat\varepsilon_j\geq\varepsilon$,
the optimal assortment $S^*$ is never removed from their active item sets with high probability.
\begin{lemma}
With probability $1-O(\tau_0 NJ/T^2)$ it holds for all $\tau$ and $\hat\varepsilon_j\geq\varepsilon$ that $S^*\subseteq\mathcal A_j^{(\tau)}$.
\label{lem:feasible-adaptive}
\end{lemma}
Lemma \ref{lem:feasible-adaptive} is similar in spirit to the structural results established in Lemma \ref{lem:feasible}
for Algorithm \ref{alg:active-elimination}, but it is only applicable to thread $j$ with $\hat\varepsilon_j \geq \varepsilon$
The remaining threads, with $\hat\varepsilon_j < \varepsilon$ are too aggressive in their elimination strategy, so we cannot guarantee that $S^\star \subseteq \mathcal{A}_j^{(\tau+1)}$ for all $\tau$. 
We will see how to upper bound the regret from these threads later in this section. 



Our next lemma analyzes the Step \ref{step:stopping} of the algorithm:
\begin{lemma}
\label{lem:J-lowerbound}
If $\hat\varepsilon_J\geq \varepsilon$ then with probability $1-O(\tau_0 NJ/T)$,
Algorithm \ref{alg:adaptive-epsilon} will not be re-started.
\end{lemma}

At a high level, Lemma \ref{lem:J-lowerbound} states if step \ref{step:stopping} is triggered
(which causes $J\gets J-1$ and a re-start of the entire algorithm), the smallest hypothetical value $\hat\varepsilon_J$ must be below the actual value of $\varepsilon$. 
First, this ensure that the algorithm does not restart too often, but more importantly, it guarantees that the actual $\varepsilon$ always falls between $\hat\varepsilon_0$ and $\hat\varepsilon_J$ throughout the entire selling period.

The proof of Lemma \ref{lem:J-lowerbound} is based on Lemma \ref{lem:feasible-adaptive}.
In particular, the condition in Step \ref{step:stopping} of Algorithm \ref{alg:adaptive-epsilon} compares the optimistic assortments $S_{\tau,j}^{(i)}$
in thread $j$
with estimates in threads $j'<j$, which have larger $\hat\varepsilon_j$ values.
If, hypothetically, $\hat\varepsilon_j$ is larger than or equal to $\varepsilon$,
then by Lemma \ref{lem:J-lowerbound}, we know that $S^*\subseteq \mathcal A_{j'}^{(\tau+1)}$ for all $j'\leq j$,
and therefore the estimated optimality of $S_{\tau,j}^{(i)}$ should be consistent in all threads $j'\leq j$.
Hence, any inconsistency detected by step \ref{step:stopping} must imply that  $\hat\varepsilon_j<\varepsilon$,
which justifies decreasing $J$. 

We now present two lemmas that upper bound the regret accumulated by different threads, which requires some new notation. 
For $0\leq j< J$, let $\sR(\hat\varepsilon_j)$ denote the cumulative regret incurred during the time periods in which thread $j$ is run.
Clearly, the total regret incurred is upper bounded by $\sum_{j < J}\sR(\hat\varepsilon_j)$. 
Using linearity of the expectation, it then suffices to upper bound $\mathbb E[\sR(\hat\varepsilon_j)]$ for every $j< J$.
The next two lemmas provide these upper bounds for two different scenarios.
For notational simplicity we use $\lesssim$ to hide $\poly(K,\log(NT))$ factors. 
\begin{lemma}
\label{lem:erj-bound1}
For all $j< J$ satisfying $\hat\varepsilon_j\geq\varepsilon$, $\mathbb E[\sR(\hat\varepsilon_j)] \lesssim \sum_{\tau\leq\tau_0}\mathbb E[\hat\Delta_{\hat\varepsilon_j}(\tau)\times \wp_jT_\tau]$.
\end{lemma}

\begin{lemma}
\label{lem:erj-bound2}
For all $j< J$ satisfying $\hat\varepsilon_j<\varepsilon$ and any $\hat\varepsilon_k>\max\{\hat\varepsilon_j,\varepsilon\}$, it holds that
$
\mathbb E[\sR(\hat\varepsilon_j)] \lesssim \sum_{\tau\leq\tau_0}\mathbb E[\hat\Delta_{\hat\varepsilon_k}(\tau)\times \wp_j T_\tau].
$
\end{lemma}


These two lemmas upper bound the total accumulated regret 
of threads $0\leq j<J$, separately for the case of $\hat\varepsilon_j\geq\varepsilon$ and $\hat\varepsilon_j<\varepsilon$.
The case of $\hat\varepsilon_j \geq\varepsilon$ is relatively straightforward to prove, since $S^*\subseteq\mathcal A_j^{(\tau+1)}$ as shown in Lemma \ref{lem:feasible-adaptive}, so an argument similar to the proof of Theorem~\ref{thm:upper-bound} applies. 
On the other hand, the case of $\hat\varepsilon_j <\varepsilon$ is more difficult because $S^*$ might be eliminated in these threads.
For Lemma~\ref{lem:erj-bound2}, which considers this case, we carefully analyze the stopping rule in Step~\ref{step:stopping}, essentially showing that the check in Step~\ref{step:stopping} will trigger as soon as the regret per-time period is too high for these threads. 
The complete proofs of both lemmas, as well as the complete proof of Theorem \ref{thm:adaptive-epsilon}, are deferred to the supplementary material.

{
\section{Instance gap-dependent analysis}
\label{sec:gap}
Recall that $S^*$ is the optimal assortment. For any given item $i$, let $S^{*,(i)} = \arg\max_{|S|\leq K, S \ni i } R(S)$ be the optimal assortment containing the specific item $i$.
Define the sub-optimality ``gap'' $\beta$ as
\begin{equation}
\beta := R(S^*) - \max_{i\notin S^*}R(S^{*,(i)}).
\label{eq:defn-gap}
\end{equation}

Intuitively, the sub-optimality gap defined in Eq.~(\ref{eq:defn-gap}) measures how ``well-defined'' the optimal assortment $S^*$ is,
in the sense that the inclusion of any \emph{non-optimal item} $i\notin S^*$ would result in at least a drop of $\beta$ in expected revenue/reward,
regardless of how other products in the assortment are selected.
If a problem instance has a large sub-optimality gap parameter $\beta$, it implies that the optimal assortment $S^*$ is easier to learn
(since non-optimal products are easier to be ruled out) and therefore smaller cumulative regret is expected.

It is also worthwhile to compare the gap parameter defined in Eq.~(\ref{eq:defn-gap}) with those defined in earlier works.
In the work of \cite{Rusmevichientong2010}, a non-parametric gap $\beta'$ is defined as
$$
\beta' := \frac{\min\{\min_i v_i, \min_{i\neq j}|v_i-v_j|, \min_{(i,j)\neq(s,t)}|\mathcal J(i,j)-\mathcal J(s,t)|\}}{(1+K\max_i v_i)},
$$
where $\mathcal J(i,j) := (r_iv_i-r_jv_j)/(v_i-v_j)$.
It is clear that a strictly positive $\beta'$ implies that all utility parameter $\{v_i\}$ are distinct.
On the other hand, it is easy to construct problem instances with duplicate $v_i$ parameters (indicating that some products have the same utility/popularity
for incoming customers) and zero $\beta'$, while our defined sub-optimality gap $\beta$ could still be strictly positive.
Indeed, consider the following problem instance with $N=3$ products and $K=2$ capacity constraint, with $(v_1,v_2,v_3)=(0.5,0.5,1)$ and
$(r_1,r_2,r_3) = (0.2,0.5,0.6)$. It is easy to verify that in this problem instance $\beta'=0$, while $\beta=0.06>0$.

In the remainder of this section,
we will use the concept of sub-optimality gap defined in Eq.~(\ref{eq:defn-gap}) to improve our regret upper bounds
in Theorems \ref{thm:upper-bound} and \ref{thm:adaptive-epsilon},
resembling $\log-T$ type gap-dependent regret bounds in stochastic mutli-armed bandits.
Both our Algorithms \ref{alg:active-elimination} and \ref{alg:adaptive-epsilon} remain unchanged, while the regret analysis
is modified to take into consideration the $\beta$ parameter.

\subsection{Gap-dependent analysis of Algorithm \ref{alg:active-elimination} (known corruption level)}

We first consider Algorithm \ref{alg:active-elimination} designed for the setting in which a good upper bound $\bar\varepsilon$
on the true corruption level $\varepsilon$ is known. The following lemma is the key lemma in the gap-dependent setting: 
\begin{lemma}
Let $\beta$ be defined in Eq.~(\ref{eq:defn-gap}) and suppose $\beta>0$.
Then with probability $1-O(\tau_0 N/T^2)$, for every epoch $\tau$ satisfying 
\begin{equation} 
T_\tau \geq \kappa_0\times \max\left\{\frac{\bar\varepsilon K^2T}{\beta}, \frac{K^2\sqrt{\bar\varepsilon NT\log T}}{\beta},
\frac{K^2 N\log T}{\beta},  \frac{KN\log T}{\beta^2}\right\},
\label{eq:known-eps-early-stop}
\end{equation}
for some universal constant $\kappa_0>0$, 
it holds that $\mathcal A^{(\tau+1)} = S^*$.
\label{lem:known-eps-early-stop}
\end{lemma}

We note that in \eqref{eq:known-eps-early-stop}, $\bar\varepsilon$ is an upper bound estimate of $\varepsilon$. At a high level, Lemma \ref{lem:known-eps-early-stop} states that if $T_\tau$ is sufficiently large, the active product set $\mathcal A^{(\tau)}$
only consists of the optimal assortment for typical customers $S^*$.
Intuitively, this is because when $T_\tau$ is large, the confidence bound $\hat\Delta_{\bar\varepsilon}(\tau)$ is much shorter.
When the confidence interval cannot cover the underlying sub-optimality gap $\beta$, the non-optimal products $i\notin S^*$ will be automatically eliminated.
A complete proof of Lemma \ref{lem:known-eps-early-stop} is given in the supplementary material.

With Lemma \ref{lem:known-eps-early-stop}, we can prove the following theorem on gap-dependent regret upper bounds
for Algorithm \ref{alg:active-elimination} with a known upper bound $\bar\varepsilon$ on $\varepsilon$.

\begin{theorem}
Let $\beta$ be defined in Eq.~(\ref{eq:defn-gap}) and $\beta>0$.
Assume also for simplicity that $\bar\varepsilon\lesssim 1/K^3$.
The expected cumulative $\tote$-regret of Algorithm \ref{alg:active-elimination} is upper bounded by
\begin{equation}\label{eq:gap_upper}
C_0'\times\left(\bar\varepsilon K^2 T\log T + \frac{K^2N\log^2 T}{\beta}\right)
\end{equation}
where $C_0'<\infty$ is a universal constant.
\label{thm:known-eps-gap-dependent}
\end{theorem}

We remark that the $\log^2T$ term in the second $\frac{K^2 N\log^2 T}{\beta}$ term in the regret upper bound
most likely arises from the doubling epochs $\{\mathcal A^{(\tau)}\}$ used in our proposed active elimination algorithms, 
where the total number of epochs $\tau_0$ could be logarithmic in $T$. 
It is an interesting open technical question to further improve the second term in \eqref{eq:gap_upper} to be linear in $\log T$,
which should be possible at least in the case of $\varepsilon$ (or its suitable upper bound $\bar\varepsilon$) being known.


\subsection{Gap-dependent analysis of Algorithm \ref{alg:adaptive-epsilon} (unknown corruption level)}

When the corruption level $\varepsilon$ is unknown and no good estimate is available a priori, 
Algorithm \ref{alg:adaptive-epsilon} partitions the possible corruption levels into a logarithmic grid $\{\hat\varepsilon_j\}_{j=0}^{J-1}$
and runs Algorithm \ref{alg:active-elimination} on different levels of $\hat\varepsilon_j$ in parallel.
To analyze its regret performance from a gap-dependent perspective, we again discuss the two cases of $\hat\varepsilon_j\geq \varepsilon$
and $\hat\varepsilon_j<\varepsilon$ separately.

In the case of $\hat\varepsilon_j\geq\varepsilon$ (i.e., over-estimating the true corruption level $\varepsilon$), 
Lemma \ref{lem:feasible-adaptive} shows that with high probability the optimal assortment $S^*$ will not be removed from $\mathcal A_j^{(\tau)}$.
Subsequently, Lemma \ref{lem:known-eps-early-stop} can be directly applied,
with a union bound on the failure probability over $j<J$, $\hat\varepsilon_j\geq\varepsilon$, as the following corollary:
\begin{corollary}
For $j<J$ and epoch $\tau$ recall the definitions that $T_j=\wp_j T$ and $T_{\tau,j}=\wp_jT_\tau$,
where $\wp_j=2^{-(J-j)}/(1-2^{-J})$ is the sampling probability for thread $j$ and $T_\tau=2^\tau T_0$ is the ``normal'' length epoch $\tau$.
Let $\tau_j^*$ be the smallest integer such that $T_{\tau_j^*,j}$ satisfies Eq.~(\ref{eq:known-eps-early-stop}), or more specifically
\begin{equation}
T_{\tau_j^*,j} \geq \kappa_0'\times \max\left\{\frac{\hat\varepsilon_j K^2T_j}{\beta}, \frac{K^2\sqrt{\hat\varepsilon_j NT_j\log T}}{\beta},
\frac{K^2 N\log T}{\beta},  \frac{KN\log T}{\beta^2}\right\},
\label{eq:unknown-eps-early-stop}
\end{equation}
where $\kappa_0'>0$ is a universal constant. Then for all $\tau'\geq\tau_j^*$, $\mathcal A_j^{(\tau)}=S^*$.
\label{cor:unknown-eps-early-stop}
\end{corollary}

Subsequently, Lemma \ref{lem:erj-bound1} leads to the following corollary:
\begin{corollary}
For all $j<J$ satisfying $\hat\varepsilon_j\geq\varepsilon$, $\mathbb E[\mathsf R(\hat\varepsilon_j)] \lesssim 
\mathbb E[\sum_{\tau\leq\tau_j^*}\hat\Delta_{\hat\varepsilon_j}(\tau)\times \wp_j T_\tau]$,
where $\tau_j^*$ is defined in Corollary \ref{cor:unknown-eps-early-stop}.
\label{cor:erj-bound1-gap}
\end{corollary}

We next consider the case of $\hat\varepsilon_j<\varepsilon$.
Because the constraint $\mathcal A_{j+1}^{(\tau)}\subseteq\mathcal A_j^{(\tau)}$ is enforced in Algorithm \ref{alg:adaptive-epsilon}
all the time, we know that $\mathcal A_j^{(\tau)}= S^*$ implies $\mathcal A_{j+1}^{(\tau)}=S^*$ with probability 1.
Consequently, Lemma \ref{lem:erj-bound2} implies the following:
\begin{corollary}
For all $j<J$ satisfying $\hat\varepsilon_j<\varepsilon$
and any $\hat\varepsilon_k>\max\{\hat\varepsilon_j,\varepsilon\}$, it holds that
$
\mathbb E[\sR(\hat\varepsilon_j)] \lesssim \mathbb E[\sum_{\tau\leq\tau_k^*}\hat\Delta_{\hat\varepsilon_k}(\tau)\times \wp_j T_\tau],
$
where $\tau_k^*$ is defined in Corollary \ref{cor:unknown-eps-early-stop} for thread $k$.
\label{cor:erj-bound2-gap}
\end{corollary}

With Corollaries \ref{cor:unknown-eps-early-stop}, \ref{cor:erj-bound1-gap} and \ref{cor:erj-bound2-gap} in place,
we are ready to state our gap-dependent analysis for Algorithm \ref{alg:adaptive-epsilon} with unknown corruption level $\varepsilon$.
\begin{theorem}
Suppose Algorithm \ref{alg:adaptive-epsilon} runs with an initial value of $J=\lfloor \log_2(\sqrt{N/T})\rfloor + 1$.
Suppose also that the gap parameter $\beta$ defined in Eq.~(\ref{eq:defn-gap}) is strictly positive. 
Then the cumulative $\tote$-regret of Algorithm \ref{alg:adaptive-epsilon}
can be upper bounded by 
$$
\left(\varepsilon T + N/\beta^2\right) \times \poly(K, \log(NT)),
$$
where in the regret upper bound we hide polynomial dependency on $K$ and $\log N,\log T$ terms.
\label{thm:adaptive-epsilon-gap}
\end{theorem}
\begin{remark}
An alternative upper bound of $(\varepsilon T/\beta + N/\beta)\times\poly(K,\log(NT))$ can also be proved,
which could be larger or smaller than the one presented in Theorem \ref{thm:adaptive-epsilon-gap} depending on the values of
$\varepsilon$ and $\beta$.
\end{remark}
 
Comparing Theorem \ref{thm:adaptive-epsilon-gap} with Theorem \ref{thm:known-eps-gap-dependent},
we notice an additional $1/\beta$ term in either the $\varepsilon T$ or the $N/\beta$ term in Theorem \ref{thm:known-eps-gap-dependent}.
Such a worsened dependency likely arises from the layered approach taken to address unknown $\varepsilon$ values,
which also delivered sub-optimal regret guarantees (compared to when $\varepsilon$ is known a priori)
in robust multi-armed bandit problems \citep{Lykouris:18,gupta2019better}.

\subsection{A lower bound on gap-dependent regret}\label{subsec:lower-bound-gap-dependent}

We complement our gap-depednent regret upper bound results in the previous sections
by stating a lower bound on gap-dependent regret
in dynamic assortment optimization with outlier customers.

\begin{theorem}
Let $K,\beta$ be constants independent of $T$, satisfying $\beta\leq\min\{1/16, 1/K\}$ and $K\leq 2$.
Suppose also that $\varepsilon,N$ can potentially change with $T$, and that $\beta\geq\sqrt{N/T}$, $K<N/4$.
Then for sufficiently large $T$, the worst-case $\bih$-regret of any admissible policy is lower bounded by 
$$
c_0'\times \left(\min\{\varepsilon T, \sqrt{\varepsilon NT}\} + \frac{N\log T}{K\beta}\right),
$$
where $c_0'>0$ is a universal constant independent of $N,T,K$ and $\beta$.
\label{thm:gap-dependent-lower-bound}
\end{theorem}
\begin{remark}
The lower bound result in Theorem \ref{thm:gap-dependent-lower-bound} assumes the algorithm
has full knowledge of the corruption level $\varepsilon$.
\end{remark}
\begin{remark}
	\label{rem:lower}
As Theorem \ref{thm:gap-dependent-lower-bound} only concerns the $\bih$-regret,
a similar lower bound for the $\tote$-regret can be established.
More specifically, the $\Omega(\varepsilon T)$ lower bound in Proposition \ref{prop:lowerbound} still applies,
because there is no additional constraints/assumptions imposed on outlier customers.
Furthermore, the $\frac{N\log T}{K\beta}$ lower bound in Theorem \ref{thm:gap-dependent-lower-bound} is obtained
by simply setting $\varepsilon=0$, which applies to the $\tote$-regret notion too.
Hence, a lower bound of $\Omega(\varepsilon T + \frac{N\log T}{K\beta})$ can be established for the $\tote$-regret in the gap-dependent setting.
\end{remark}

Comparing Theorem \ref{thm:gap-dependent-lower-bound} to Theorem \ref{thm:known-eps-gap-dependent}
(our regret upper bound with knowledge of $\varepsilon$), we notice that the $K^2 N\log^2 T/\beta$ term matches
the $N\log T/(K\beta)$ term in Theorem \ref{thm:gap-dependent-lower-bound} up to polynomial dependency on $K$ and $\log T$.
As discussed in the works of \cite{Agrawal17MNLBandit,Agrawal17Thompson}, in revenue management applications
the capacity constraint $K$ is usually very small and therefore treated as a constant.
On the other hand, there is a gap between the {$\varepsilon K^2 T \log T$} term in the upper bound and the 
$\min\{\varepsilon T,\sqrt{\varepsilon NT}\}$ term in the lower bound, {particularly when $\varepsilon$ is relatively large compared to $N/T$.} 
We are at the moment unsure which one is tight. However, in order for the lower bound to be tight, it requires fully-adversarial algorithms
for dynamic assortment optimization, which has already been an open question as discussed before.
Finally, the lower bound in Theorem \ref{thm:gap-dependent-lower-bound} assumes the knowledge of the corruption level $\varepsilon$.
The lower bound for cases when $\varepsilon$ is unknown is significantly more complicated and could involve 
whether the upper bounds are tight in $\log T$ terms and the distinction between regret and pseudo-regret notions \citep{Lykouris:18},
which is out of the scope of this paper.
}



\section{Numerical illustration}
\label{sec:numerical}
\begin{figure}[!t]
	\centering
	\includegraphics[width=0.45\textwidth]{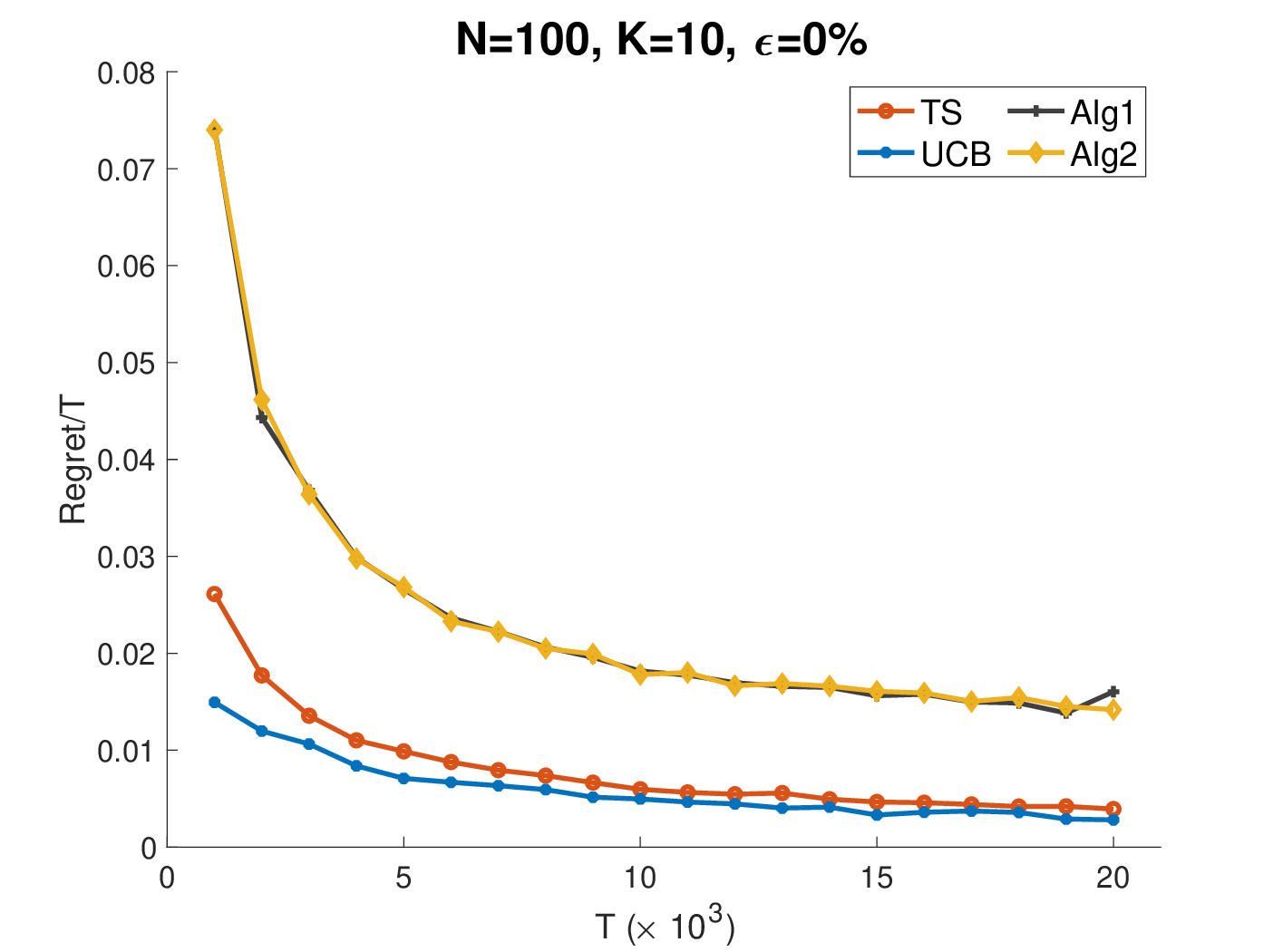}
	\includegraphics[width=0.45\textwidth]{result_N100K10veps0.eps}
	\includegraphics[width=0.45\textwidth]{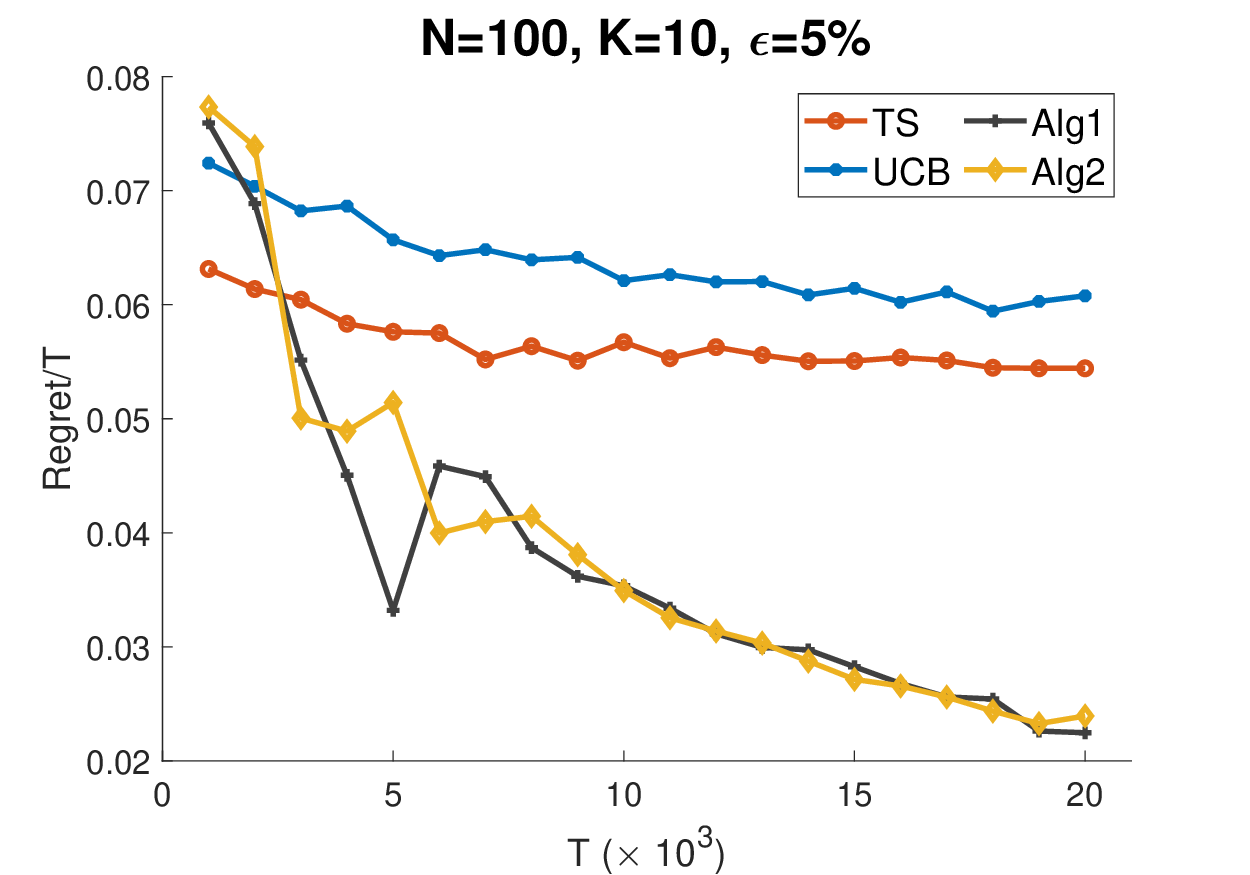}
	\includegraphics[width=0.45\textwidth]{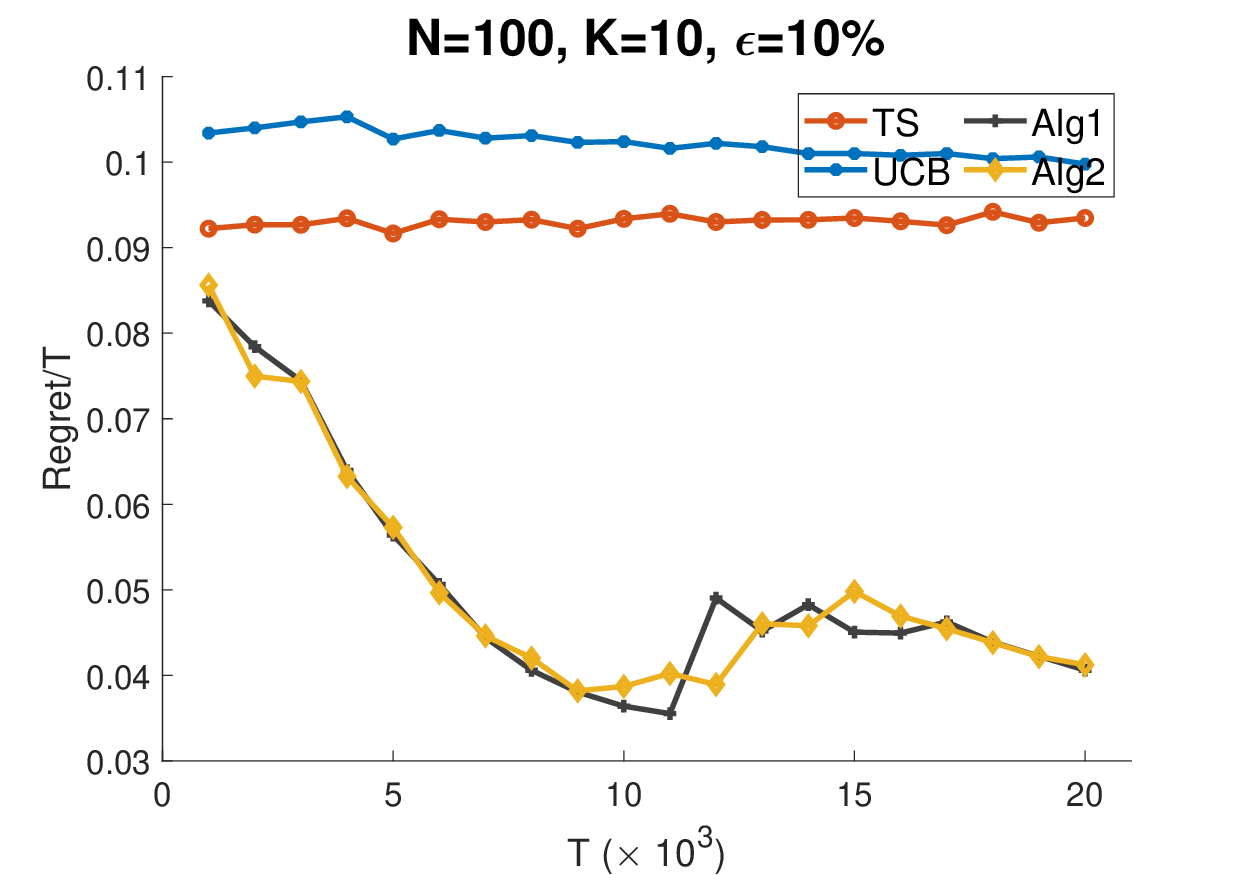}
	\includegraphics[width=0.45\textwidth]{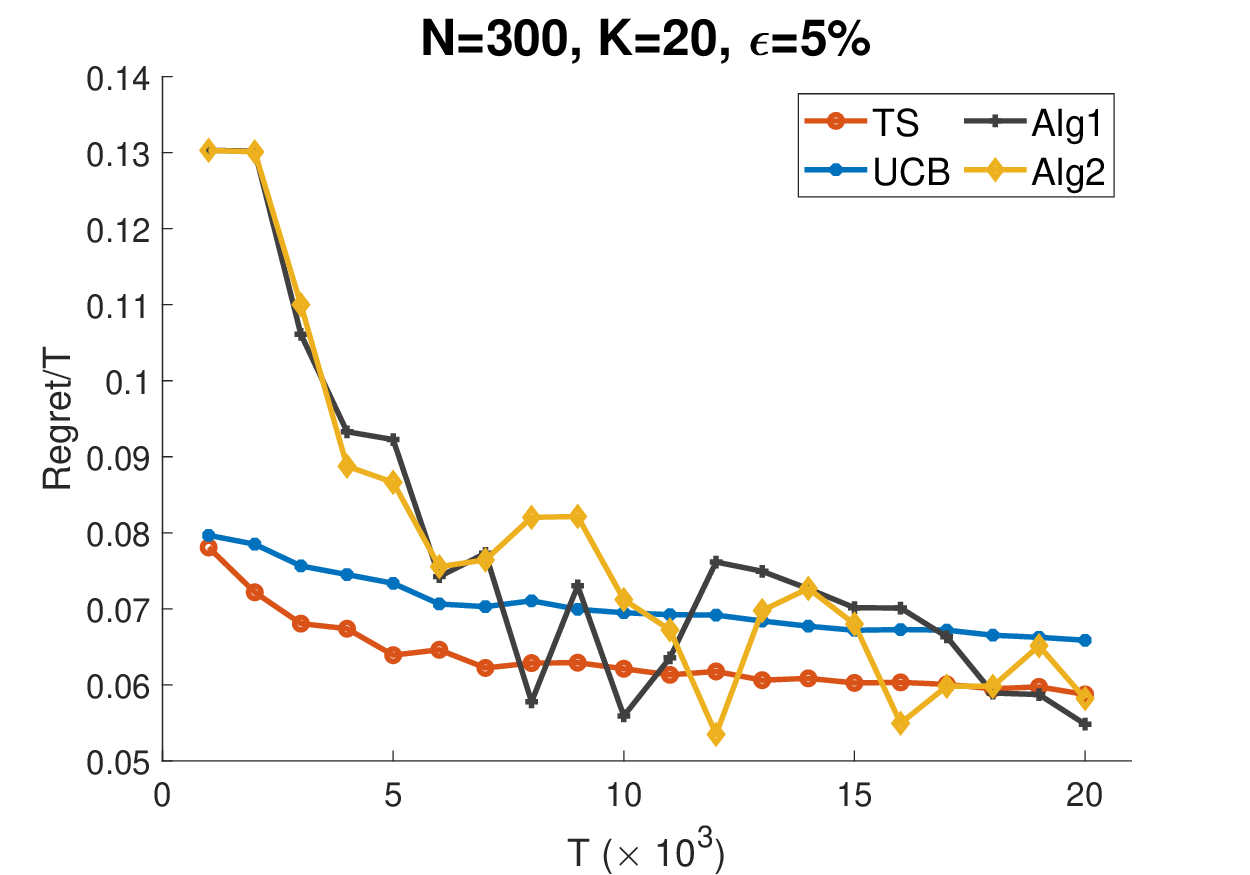}
	\includegraphics[width=0.45\textwidth]{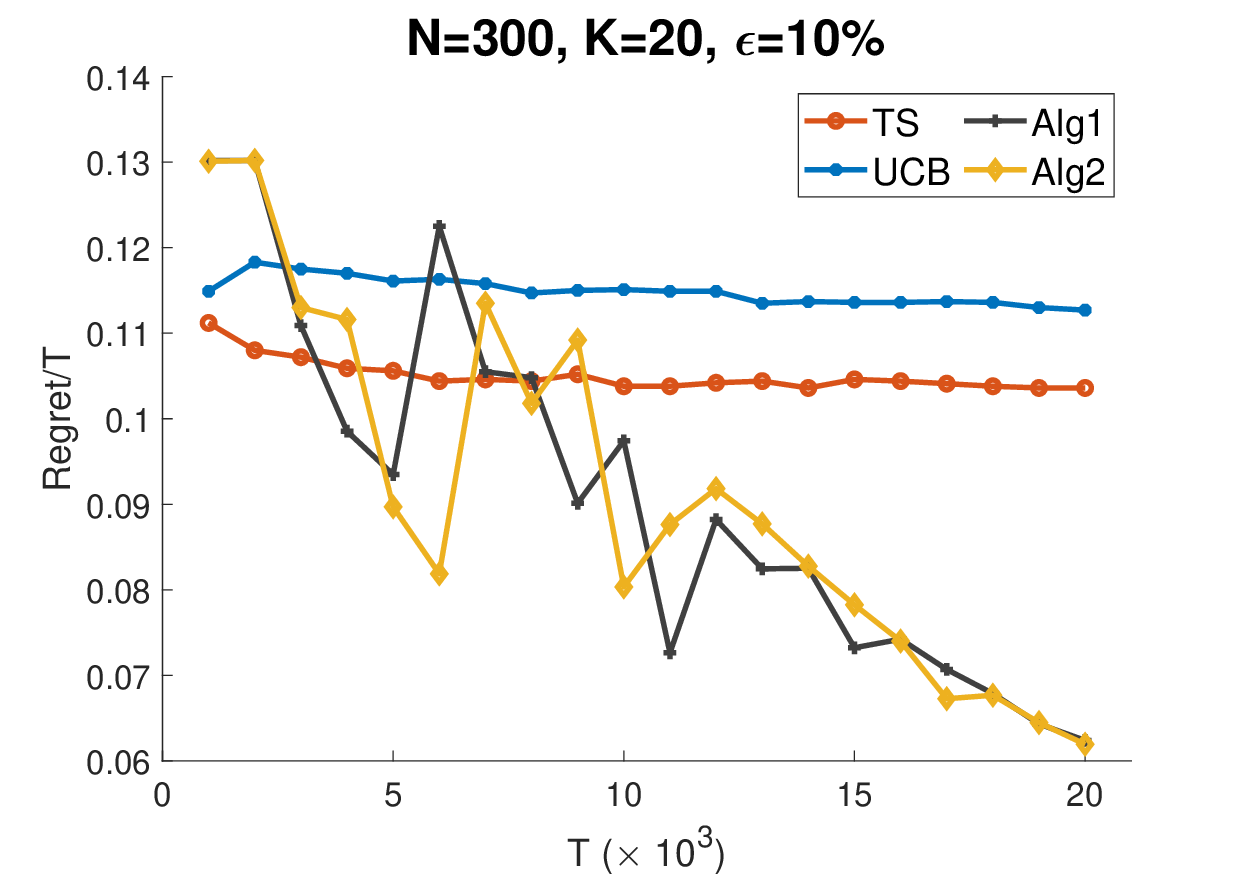}
	\caption{Comparison of average regret (i.e., regret divided by $T$) between our proposed algorithms and baselines.
		The time horizon $T$ ranges from 1,000 to 20,000. 
	}
	\label{fig:varyT}
\end{figure}

In this brief experimental section, we provide some numerical
illustrations that demonstrate the robustness of our proposed policy and the benefits over existing non-robust approaches for dynamic assortment optimization, including Thompson Sampling (TS) \citep{Agrawal17Thompson} and Upper Confidence Bounds (UCB) \citep{Agrawal17MNLBandit}. 
We construct the following data instance:
\begin{enumerate}
\item $K$ out of $N$ items have revenue parameters $r_i\equiv 1$ and preference parameters $v_i\equiv 0$;
\item For the other $(N-K)$ items, both their revenue and preference parameters $(r_i,v_i)$ are uniformly distribution on $[0.1,0.2]$;
\item For the first $\lfloor \varepsilon T\rfloor$ time periods, the arriving customers are outliers with choice models $Q_t\equiv Q$,
where $Q$ is an MNL-parameterized choice model with preference parameters set as $v_i'=1$ if $v_i=0$ and $v_i'=v_i$ otherwise.
\end{enumerate}

This instance reflects two important properties of outlier customers in practice, namely that they have significantly different preferences from typical customers, and that they arrive in consecutive time periods (e.g., during a holiday season).
In particular, the instance consists of $K$ items with very high revenue, but very low preference parameters
so that few customers will buy them. 
Under normal circumstances, a dynamic assortment optimization algorithm would identify the unpopularity of these $K$ items very quickly and stop recommending them.
However, as the outlier customers prefer these $K$ items over the others, these items appear popular and profitable in the early time periods, which may mislead the algorithm.
As these algorithms are highly unpopular in the latter time periods, a robust algorithm should not be severely impacted by these outlier customers. 





For the baseline methods, the TS method is tuning-free with a non-informative $\mathrm{Beta}(1,1)$ prior on each item.
For the UCB algorithm, we find the value in the multiplier ($C_1$) when constructing upper confidence bands that gives the best performance (in the original paper of \cite{Agrawal17MNLBandit}
$C_1=48$ for theoretical purposes).  
Each method is run for 100 independent trials and the mean average regret (i.e., the cumulative regret over $T$) is reported. 
The standard deviation of all the methods are sufficiently small and thus omitted for better visualization.



In Figure \ref{fig:varyT}, we report the results for all methods under various settings of $T,N,K$ and $\varepsilon$.
The experimental settings are chosen as $N\in\{100,300\}$, $K\in\{10,20\}$, $\varepsilon\in\{0,0.05,0.1\}$ and $T$ ranging from $T=1,000$ to $T=20,000$.
From Figure \ref{fig:varyT}, we can see that when $\varepsilon$ is strictly greater than 0,
our proposed algorithms will stabilize at a mean regret level (0.02 to 0.06) that is much lower than the non-robust TS and UCB methods.
More importantly, the average regret (i.e., cumulative regret divided by $T$) for our method decreases as a function of the time horizon, a phenomenon that does not happen for TS/UCB, especially when $\varepsilon$ is large. 
This confirms that these latter two methods are \emph{not} robust to outlier customers, and further confirms the effectiveness of our proposed algorithms for robust dynamic assortment optimization. For the  no contamination case of $\varepsilon=0$, while our proposed algorithms perform slightly worse than the baselines,   the decreasing rates of average regrets are the same.  		  	When there is no contamination, although the main term in our regret $\sqrt{NT}$  is still tight, there  might  be extra overhead in the regret bound through dependency on $K$ and $\log T$ factors.


\section{Conclusions and Future Work}
\label{sec:conclusion}
In this paper, we extend the $\varepsilon$-contamination model from statistics to the online decision-making setting and study the dynamic assortment optimization problem with outlier customers.
We propose a new active elimination policy that is robust to adversarial corruptions and establish a near-optimal regret bound. 
We further develop an adaptive policy that does not require any prior knowledge of the corruption proportion $\varepsilon$.


{The first interesting problem is to sharpen upper and lower regret bounds in the gap-dependent case. Moreover, it is technically interesting to further extend the paper to the fully adversarial MNL bandit setting.} 
Beyond this technical question, we hope that this work inspires future work on model mis-specification in revenue management, which we believe is a practically important research direction. 
We look forward to pursuing this direction in future work.

\bibliographystyle{informs2014}
\bibliography{refs}





\ECSwitch

\ECHead{Supplementary Material: Additional Proofs}

\section{Proofs of the lower bounds}

\subsection{Proof of Proposition \ref{prop:lowerbound}}

\renewcommand{\theproposition}{\ref{prop:lowerbound}}
\begin{proposition}[restated]
There exists a universal constant $c_0>0$ such that, for any policy $\pi$, its worst-case regret for problem instances with $T$ customers,
$N$ items, $K<N/4$ assortment capacity constraint and $\lfloor\varepsilon T\rfloor$ outlier customers ($0\leq\varepsilon<1$) is lower bounded by 
$c_0\times (\varepsilon T + \sqrt{NT})$.
\end{proposition}

\begin{proof}
It suffices to prove that the maximum regret of the worst-case problem of any $\pi$ is lower bounded by $\Omega(\max\{\varepsilon T,\sqrt{NT}\})$, because $\varepsilon T+\sqrt{NT}\leq 2\max\{\varepsilon T,\sqrt{NT}\}$.
An $\Omega(\sqrt{NT})$ lower bound has already been established in \citep{Chen:18tight} wit no outlier customers.
Hence to prove Proposition \ref{prop:lowerbound} we only need to establish an $\Omega(\varepsilon T)$ regret lower bound.

Consider two problem instances $P_1=\{r_i,v_i\}_{i=1}^n$ and $P_2=\{r_i,v_i'\}_{i=1}^n$ with shared revenue parameters $\{r_i\}_{i=1}^n$
and different preference parameters (for typical customers) $\{v_i\}_{i=1}^n$, $\{v_i'\}_{i=1}^n$, such that for any assortment $S\subseteq[N]$, $|S|\leq K$,
$\max\{R(S^*_1|P_1)-R(S|P_1), R(S^*_2|P_2)-R(S|P_2)\} = \Omega(1)$, where $S_1^*$ and $S_2^*$ are the optimal assortments under $P_1$ and $P_2$, respectively.
The existence and explicit construction of such problem instances can be found in \citep{Chen:18tight}.
Now consider the case in which all of the first $\lfloor\varepsilon T\rfloor$ customers are outliers, associated with the \emph{same} outlier choice model $Q$ under both $P_1$ and $P_2$.
Because the choice model of the outlier customers are the same, no algorithm can distinguish $P_1$ from $P_2$ during the first $\lfloor \varepsilon T\rfloor$ time periods with success probability larger than $1/2$.
Therefore, the worst-case regret (under $P_1$ and $P_2$) of any algorithm is at least $\Omega(\varepsilon T)$, which is to be demonstrated.
\end{proof}

\section{Proofs of technical lemmas for Theorem \ref{thm:upper-bound}}

\subsection{Proof of Lemma \ref{lem:ucb-lcb}}

\renewcommand{\thelemma}{\ref{lem:ucb-lcb}}
\begin{lemma}[restated]
Suppose $T_0\geq 128(K+1)^2N_\tau \ln T$ and $\min\{1,\varepsilon T/T_\tau\}\leq 1/4(K+2)$.
With probability $1-O(\tau_0 N/T^2)$ it holds for all $\tau$ satisfying $T_\tau\geq\max\{\overline\varepsilon,\varepsilon\} T/4(K+1)$ and $i\in\mathcal A^{(\tau+1)}$ that 
$|\hat v_i^{(\tau+1)}-v_i|\leq \Delta_\varepsilon^*(i,\tau+1)$, where
\begin{equation}
\Delta_\varepsilon^*(i,\tau+1) = 
8(K+1)\left(\frac{\varepsilon_\tau}{2} + \sqrt{\frac{\varepsilon_\tau N_\tau\ln T}{T_\tau}} + \frac{2N_\tau \ln T}{3T_\tau}\right) + 8\sqrt{\frac{(1+V_S)v_iN_\tau\ln T}{T_\tau}},
\end{equation}
where $\varepsilon_\tau$ is defined as $\varepsilon_\tau=\min\{1,\varepsilon T/T_\tau\}$, $N_\tau=|\mathcal A^{(\tau+1)}|$ and $V_S= \sum_{j\in S_\tau^{(i)}}v_j$.
\end{lemma}

\begin{proof}
%
Denote $N_\tau := |\mathcal A^{(\tau+1)}|$ and let $\mathcal T_\tau$ be the $T_\tau$ consecutive time periods during which assortments $S_\tau^{(i)}$, $i\in\mathcal A^{(\tau+1)}$
are offered uniformly at random.
Let also $\mathcal T_\tau$ be the set of all time periods at epoch $\tau$.
For each $i\in\mathcal A^{(\tau+1)}$ and $t\in\mathcal T_\tau$, define indicator variable $I_{ti}(0)=1$ if $S_\tau^{(i)}$ is offered at time $t$ and the no-purchase action $i_t=0$ is taken from the incoming customer,
and $I_{ti}(0)=0$ otherwise.
Similarly, define $I_{ti}(i)=1$ if $S_\tau^{(i)}$ is offered at time $t$ and the purchase of item $i$, $i_t=i$, is observed from the incoming customer.
We then have, by definition, that, 
\begin{equation}
\frac{n_0(i)}{T_\tau} = \frac{1}{T_\tau}\sum_{t\in\mathcal T_{\tau}} I_{ti}(0), \;\;\;\;\;
\frac{n_i}{T_\tau} = \frac{1}{T_\tau}\sum_{t\in\mathcal T_{\tau}}I_{ti}(i).
\label{eq:empirical-indicator}
\end{equation}

Recall the definition that $V_S := \sum_{j\in S_\tau^{(i)}}v_j$.
Let $\mathcal T_{\tau}^*$ and $\mathcal T_{\tau}^o$ be the time periods corresponding to typical and outlier customers, respectively.
The expectation of $n_0(i)/T_\tau$ can subsequently be calculated as
\begin{align*}
\mathbb E\bigg[\frac{n_0(i)}{T_\tau}\bigg]
&= \frac{1}{T_\tau}\left(\sum_{t\in\mathcal T_{\tau}^*}\frac{1}{N_\tau(1+V_S)} + \sum_{t\in\mathcal T_{\tau}^o}\frac{\Pr[i_t=0|Q_t]}{N_\tau}\right)\\
&= \frac{1}{N_\tau(1+V_S)} + \frac{1}{N_\tau T_\tau}\sum_{t\in\mathcal T_{\tau}^o}\left(\Pr[i_t=0|Q_t]-\frac{1}{1+V_S}\right).
\end{align*}
Note that since $i$ at time $t$ is selected uniformly at random, the events ``$S_\tau^{(i)}$ is offered at time $t$'' and ``customer arriving at time $t$ is an outlier'' are independent, 
which is crucial in the above derivation.
Define $\alpha_0 := \frac{1}{|\mathcal T_{\tau}^o|}\sum_{t\in\mathcal T_\tau^o}(\Pr[i_t=0|Q_t]-1/(1+V_S))$. 
It is clear by definition that $|\alpha_0|\leq 1$.
Furthermore, because at most $\varepsilon T$ customers are outliers throughout the entire $T$ time periods, we know that $|\mathcal T_{\tau}^o|\leq \min\{T_\tau,\varepsilon T\}$
and hence $\tilde\varepsilon_\tau := |\mathcal T_{\tau}^o|/T_\tau \leq \min\{1, \varepsilon T/T_\tau\} = \varepsilon_\tau$.
Subsequently, we have
\begin{equation}
\mathbb E\bigg[\frac{n_0}{T_\tau}\bigg] = \frac{1}{N_\tau(1+V_S)} + \frac{\tilde\varepsilon_\tau\alpha_0}{N_\tau}.
\label{eq:mean-n0}
\end{equation}
Similarly, for $n_i$ we have
\begin{equation}
\mathbb E\bigg[\frac{n_i}{T_\tau}\bigg] = \frac{v_i}{N_\tau(1+V_S)} + \frac{\tilde\varepsilon_\tau\alpha_i}{N_\tau},
\end{equation}
where $\alpha_i = \frac{1}{|\mathcal T_{\tau}^o|}\sum_{t\in\mathcal T_\tau^o}(\Pr[i_t=i|Q_t]-v_i/(1+V_S))$ which also satisfies $|\alpha_i|\leq 1$.


{
It is easy to verify that the partial sums $\sum_{t\in\mathcal T_\tau,t<s} I_{ti}(j)-\mathbb E[I_{ti}(j)|\mathcal F_{s-1}]$ form \emph{martingales}
where the filtration is $\mathcal F_{s-1} = \{S_{s'},i_{s'}| s'<s\}$,
for both $j\in\{0,i\}$, because the decision of whether customer $t$ is an outlier is independent from the event $\{S_t=S_\tau^{(i)}\}$.
}
The variances of $n_i$ and $n_i$ (conditioned on the assortments offered in epoch $\tau$) can also be upper bounded as
$\mathbb V[n_0] \leq T_\tau\varepsilon_\tau /N_\tau+ T_\tau/N_\tau(1+V_S)$ and $\mathbb V[n_i] \leq T_\tau\varepsilon_\tau/N_\tau+ T_\tau v_i/N_\tau(1+V_S)$.
Subsequently, invoking Bernstein's inequality (Lemma \ref{lem:bernstein-tech}) 
{with $M=1$, $\delta=1/T^2$ and $V^2=\varepsilon_\tau/(N_\tau T_\tau)+1/(N_\tau T_\tau(1+V_S))$ for $n_0/T_\tau$,
$V^2 = \varepsilon_\tau/(N_\tau T_\tau) + v_i/(N_\tau T_\tau(1+V_S))$ for $n_i/T_\tau$,}
we have that
$$
\frac{n_0}{T_\tau} = \frac{\tilde\varepsilon_\tau\alpha_0}{N_\tau} + \frac{1}{N_\tau(1+V_S)} + \eta_0, \;\;\;\;
\frac{n_i}{T_\tau} = \frac{\tilde\varepsilon_\tau\alpha_i}{N_\tau} + \frac{v_i}{N_\tau(1+V_S)} + \eta_i, 
$$
where
\begin{align*}
&\Pr\left[|\eta_0|> \frac{4\ln T}{3T_\tau} + 2\sqrt{\frac{\varepsilon_\tau\ln T}{N_\tau T_\tau}} + 2\sqrt{\frac{\ln T}{(1+V_S)N_\tau T_\tau}}\right] \leq \frac{2}{T^2};\\
&\Pr\left[|\eta_i|>\frac{4\ln T}{3T_\tau} + 2\sqrt{\frac{\varepsilon_\tau\ln T}{N_\tau T_\tau}} + 2\sqrt{\frac{v_i\ln T}{(1+V_S)N_\tau T_\tau}}\right]\leq \frac{2}{T^2}.
\end{align*}

Let us now consider the case of $\hat v_i = n_i/n_0\leq 1$.
$\hat v_i$ then admits the form of 
\begin{align*}
\hat v_i &= \frac{v_i+(1+V_S)(\tilde\varepsilon_\tau\alpha_i+\eta_iN_\tau)}{1+(1+V_S)(\tilde\varepsilon_\tau\alpha_0+\eta_0N_\tau)}
= v_i +\frac{(1+V_S)(\tilde\varepsilon_\tau\alpha_i + \eta_iN_\tau - v_i(\tilde\varepsilon_\tau\alpha_0+\eta_0N_\tau))}{1+(1+V_S)(\tilde\varepsilon_\tau\alpha_0 + \eta_0N_\tau)}.
\end{align*}
Additionally, by Hoeffding's inequality, it also holds that
\begin{equation}
\Pr\left[\max\{|\eta_0|,|\eta_i|\} > \sqrt{\frac{8\ln T}{N_\tau T_\tau}}\right] \leq \frac{4}{T^2}.
\label{eq:hoeffding-eta}
\end{equation}

Because $V_S\leq K$, $\alpha_0,\alpha_1\in[0,1]$ and $|\eta_0|,|\eta_i|\leq \sqrt{8\ln T/N_\tau T_\tau}$ with probability $1-4/T^2$, 
we have that (with probability $1-O(T^{-2})$ and a union bound)
{
\begin{align*}
\big|\hat v_i-v_i\big| 
&= \frac{(1+V_S)|\tilde\varepsilon_\tau(\alpha_i-\alpha_0) + \eta_i N_\tau - \eta_0 v_i N_\tau|}{1+(1+V_S)(\tilde\varepsilon_\tau\alpha_0 + \eta_0 N_\tau)}\\
&\leq \frac{(1+V_S)(2\varepsilon_\tau + (1+v_i)(4N_\tau\ln T/(3T_\tau) + 2\sqrt{\varepsilon_\tau N_\tau\ln T/T_\tau})
+ 2(1+\sqrt{v_i})\sqrt{N_\tau \ln T/(1+V_S )T_\tau}
}{1+(1+V_S)(-\varepsilon_\tau - |\eta_0|N_\tau)}\\
&\leq \frac{(1+V_S)(2\varepsilon_\tau + 8N_\tau \ln T/(3T_\tau) + 4\sqrt{\varepsilon_\tau N_\tau\ln T/T_\tau} + 4\sqrt{v_iN_\tau \ln T/(1+V_S)T_\tau})}{1-(K+1)\varepsilon_\tau- (K+1)\sqrt{8N_\tau \ln T/T_\tau}}.
\end{align*}
}
Provided that $\varepsilon_\tau\leq 1/4(K+2)$ 
\begin{align}
\big|\hat v_i-v_i\big|  &\leq 4(K+1)\varepsilon_\tau + 8(K+1)\sqrt{\frac{\varepsilon_\tau N_\tau\ln T}{T_\tau}} + \frac{16(K+1)N_\tau \ln T}{3T_\tau} +  8\sqrt{\frac{(1+V_S)v_iN_\tau \ln T}{T_\tau}}\nonumber\\
&= \Delta_\varepsilon^*(i,\tau+1).\nonumber
\end{align}
In the case of $n_i/n_0>1$ the definition of $\hat v_i=1$ only decreases $|\hat v_i-v_i|$.
A union bound over all $i$ and $\tau$ completes the proof of Lemma \ref{lem:ucb-lcb}.

\end{proof}

\subsection{Proof of Lemma \ref{lem:revenue-error}}

\renewcommand{\thelemma}{\ref{lem:revenue-error}}
\begin{lemma}[restated]
For any $S\subseteq[N]$, $|S|\leq K$ and $\{\hat v_i\}$, it holds that
$$
|R(S;\hat v)-R(S;v)|\leq\frac{2\sum_{i\in S}|\hat v_i-v_i|}{1+\sum_{i\in S}v_i}.
$$
\end{lemma}

\begin{proof}
Expanding the definitions of $R(S;\hat v)$ and $R(S;v)$, we have
\begin{align*}
\big|R(S;\hat v)-R(S;v)\big|
&= \left|\frac{\sum_{i\in S}r_i\hat v_i}{1+\sum_{i\in S}\hat v_i} - \frac{\sum_{i\in S}r_iv_i}{1+\sum_{i\in S}v_i}\right|\\
&= \left|\frac{(\sum_{i\in S}r_i\hat v_i)(1+\sum_{i\in S}v_i) - (\sum_{i\in S}r_iv_i)(1+\sum_{i\in S}\hat v_i)}{(1+\sum_{i\in S}\hat v_i)(1+\sum_{i\in S}v_i)}\right|\\
&\leq \frac{(\sum_{i\in S}r_i\hat v_i)(\sum_{i\in S}|v_i-\hat v_i|) + (1+\sum_{i\in S}\hat v_i)(\sum_{i\in S}r_i|\hat v_i-v_i|)}{(1+\sum_{i\in S}\hat v_i)(1+\sum_{i\in S}v_i)}\\
&\leq \frac{2\sum_{i\in S}|\hat v_i-v_i|}{1+\sum_{i\in S}v_i}.
\end{align*}
\end{proof}

\subsection{Proof of Corollary \ref{cor:rdiff}}

\renewcommand{\thecorollary}{\ref{cor:rdiff}}
\begin{corollary}[restated]
For every $\tau$ and $|S|\leq K$, $S\subseteq \mathcal A^{(\tau)}$, conditioned on the success events on epochs up to $\tau$, it holds that
$|R(S;\hat v^{(\tau)})-R(S;v)|\leq \hat\Delta_{\varepsilon}(\tau)\leq\hat\Delta_{\max\{\varepsilon,\bar\varepsilon\}}(\tau)$, where $\hat\Delta$
is defined in Algorithm \ref{alg:active-elimination}. 
\end{corollary}

\begin{proof}
Note that it suffices to prove the first inequality, because $\hat\Delta$ is a monotonically increasing function in $\varepsilon$.
Also, we only need to consider the case of $\varepsilon_{\tau-1}\leq 1/4(K+1)$, as $\hat\Delta_\varepsilon(\tau)=1$ otherwise which trivially upper bounds $|R(S;\hat v^{(\tau)})-R(S;v)|$. 
Invoking Lemma \ref{lem:revenue-error} and the upper bound $|\hat v_i^{(\tau)}-v_i|\leq \Delta_\varepsilon^*(i,\tau)$ in Lemma \ref{lem:ucb-lcb}, we have
(recall the definition that $V_S=\sum_{i\in S}v_i$)
\begin{align*}
\big|&R(S;\hat v^{(\tau)})-R(S;v)\big|\leq \frac{2\sum_{i\in S}\Delta_\varepsilon^*(i,\tau)}{1+V_S}\\
&\leq 2K\times 8(K+1)\left(\frac{\varepsilon_{\tau-1}}{2} + \sqrt{\frac{\varepsilon_{\tau-1}N_{\tau-1}\ln T}{T_{\tau-1}}} + \frac{2N_{\tau-1}\ln T}{3T_{\tau-1}}\right)\\
&\;\;\;\; + 16\sqrt{\frac{(1+V_S)N_{\tau-1}\ln T}{T_{\tau-1}}}\times \frac{\sum_{i\in S}\sqrt{v_i}}{1+V_S}\\
&\leq 16K(K+1)\left(\frac{\varepsilon_{\tau-1}}{2} + \sqrt{\frac{\varepsilon_{\tau-1}N_{\tau-1}\ln T}{T_{\tau-1}}} + \frac{2N_{\tau-1}\ln T}{3T_{\tau-1}}\right)  + 16\sqrt{\frac{(1+V_S)N_{\tau-1}\ln T}{T_{\tau-1}}}\frac{\sqrt{KV_S}}{1+V_S}\\
&\leq 16K(K+1)\left(\frac{\varepsilon_{\tau-1}}{2} + \sqrt{\frac{\varepsilon_{\tau-1}N_{\tau-1}\ln T}{T_{\tau-1}}} + \frac{2N_{\tau-1}\ln T}{3T_{\tau-1}}\right) + 16\sqrt{\frac{KN_{\tau-1}\ln T}{T_{\tau-1}}} \\
&= \hat\Delta_{\varepsilon}(\tau).
\end{align*}
Here in the third inequality we use Cauchy-Schwarz inequality on $\sum_{i\in S}\sqrt{v_i}$;
more specifically, $\sum_{i\in S}\sqrt{v_i} = \sum_{i\in S}\sqrt{v_i}\times 1 \leq \sqrt{\sum_{i\in S}1}\times \sqrt{\sum_{i\in S}v_i} \leq \sqrt{K}\times \sqrt{V_S}$,
as $|S|\leq K$ and $V_S=\sum_{i\in S}v_i$, so that $\sqrt{V_S}\leq\max\{1,V_S\}$.
\end{proof}

\subsection{Proof of Lemma \ref{lem:feasible}}

\renewcommand{\thelemma}{\ref{lem:feasible}}
\begin{lemma}[restated]
If $\overline\varepsilon\geq\varepsilon$ then with probability $1-O(\tau_0 N/T^2)$ it holds that $S^*\subseteq\mathcal A^{(\tau)}$ for all $\tau$.
\end{lemma}

\begin{proof}
We use induction to prove this lemma.
Let $\tau^*$ be the smallest integer such that $T_{\tau^*}\geq \overline\varepsilon T/4(K+1)$.
Because $\mathcal A^{(0)}=\cdots=\mathcal A^{(\tau^*)}=[N]$, the lemma clearly holds for $\tau^*$.
Next, conditioned on $S^*\subseteq\mathcal A^{(\tau)}$, we will prove that $S^*\subseteq\mathcal A^{(\tau+1)}$ with probability $1-O(T^{-2})$.

Let $\hat S$ be the solution of step (*) in Algorithm \ref{alg:active-elimination}.
{ 
If $\tau < \tau^*$, then $T_{\tau}<\overline\varepsilon T/4(K+1)$ and hence $\hat\Delta_{\overline\varepsilon}(\tau)=1$.
This means that $\mathcal A^{(\tau)}=[N]$.
}
If $\tau=\tau^*$, then $R(\hat S;v)\leq R(S^*;v)+\hat\Delta_{\overline\varepsilon}(\tau)$ because $\hat\Delta_{\overline\varepsilon}(\tau)=1$.
If $\tau>\tau^*$, we have
\begin{equation*}
\gamma^{(\tau)} = R(\hat S;\hat v^{(\tau)}) \leq R(\hat S;v) +\hat \Delta_{\varepsilon}(\tau)\leq R(S^*;v) + \hat\Delta_{\varepsilon}(\tau)
\leq R(S^*;v) + \hat\Delta_{\overline\varepsilon}(\tau),
\end{equation*}
where the first inequality holds by Corollary \ref{cor:rdiff}, and the last inequality holds by monotonicity of $\hat{\Delta}_{\varepsilon}(\tau)$.
In both cases, it holds that
\begin{equation}
\gamma^{(\tau)} \leq R(S^*;v)+\hat\Delta_{\overline\varepsilon}(\tau).
\label{eq:feasible-1}
\end{equation}

For any $i\in S^*$, let $S_\tau^{(i)}$ be the solution of ($\dagger$) of Algorithm \ref{alg:active-elimination}.
Because $S^*\subseteq\mathcal A^{(\tau)}$ by the induction hypothesis, we know that $S^*$ is a feasible solution to the optimization question of $(\dagger)$ and therefore
\begin{equation}
R(S_\tau^{(i)};\hat v^{(\tau)}) \geq R(S^*;\hat v^{(\tau)}) \geq R(S^*;v) - \hat\Delta_{\overline\varepsilon}(\tau).
\label{eq:feasible-2}
\end{equation}
Combining Eqs.~(\ref{eq:feasible-1},\ref{eq:feasible-2})
we have (with high probability) that $R(S_\tau^{(i)};\hat v)\geq \gamma^{(\tau)} - 2\hat\Delta_{\overline\varepsilon}(\tau)$,
 and hence $i\in\mathcal A^{(\tau+1)}$.
Repeat the argument for all $i\in S^*$ we have proved that $S^*\subseteq\mathcal A^{(\tau+1)}$ with high probability.
\end{proof}

\subsection{Proof of Lemma \ref{lem:regret-single-iter}}

\renewcommand{\thelemma}{\ref{lem:regret-single-iter}}
\begin{lemma}[restated]
Suppose $S^*\subseteq\mathcal A^{(\tau)}$ holds for all $\tau$. Then with probability $1-O(\tau_0 N/T^2)$, for every $\tau\leq \tau_0$ and $i\in\mathcal A^{(\tau+1)}$, it holds that
$R(S^*;v)-R(S_\tau^{(i)};v) \leq 4\hat\Delta_{\overline\varepsilon}(\tau)$.
\end{lemma}

\begin{proof}
Because $i\in\mathcal A^{(\tau+1)}$, we know that
$R(S_\tau^{(i)};\hat v^{(\tau)}) \geq \gamma^{(\tau)} - 2\hat\Delta_{\overline\varepsilon}(\tau)$.
Additionally, because $S^*\subseteq\mathcal A^{(\tau)}$, we have that
$
\gamma^{(\tau)} = R(\hat S_\tau;\hat v^{(\tau)})\geq R(S^*;\hat v^{(\tau)}) \geq R(S^*;v)-\hat\Delta_{\overline\varepsilon}(\tau),
$
with probability $1-O(\tau_0 N/T^2)$ by invoking Corollary \ref{cor:rdiff}.
Subsequently,
$$
R(S_\tau^{(i)}; v) \geq R(S_\tau^{(i)};\hat v^{(\tau)}) - \hat\Delta_{\overline\varepsilon}(\tau) \geq \gamma^{(\tau)}-3\hat\Delta_{\overline\varepsilon}(\tau)  \geq R(S^*;v)-4\hat\Delta_{\overline\varepsilon}(\tau),
$$
which is to be demonstrated.
\end{proof}

\section{Proofs of technical lemmas of Theorem \ref{thm:adaptive-epsilon}}

First, note that if $\varepsilon\lesssim\sqrt{N/T}$, the $\varepsilon T$ term in Theorem \ref{thm:adaptive-epsilon} will be dominated by the $\sqrt{NT}$ term
and is therefore not important.
Hence, throughout the rest of this section we shall assume without loss of generality that $\varepsilon\geq \sqrt{N/T}$,
which also means that $\hat\varepsilon_J\leq\varepsilon$ in the beginning.

\subsection{Proof of Lemma \ref{lem:feasible-adaptive}}

\renewcommand{\thelemma}{\ref{lem:feasible-adaptive}}
\begin{lemma}[restated]
With probability $1-O(\tau_0 NJ/T^2)$ it holds for all $\tau$ and $\hat\varepsilon_j\geq\varepsilon$ that $S^*\subseteq\mathcal A_j^{(\tau)}$.
\end{lemma}

\begin{proof}
Because each thread $j< J$ is sampled at random with probability $\wp_j$, the expected total number of outlier customers thread $j$ encounters
is upper bounded by $\varepsilon\times \wp_j T=\mathbb E[\varepsilon T_j]$.
Hence, by Bernstein's inequality and the union bound, 
for $\varepsilon\gtrsim \sqrt{N/T}$, with probability at least $1-O(T^2)$ the total number of outlier customers thread $j$ encounters is upper bounded by $O(\varepsilon T_j\log T)$. 
In the rest of this proof, we will consider $\varepsilon\to\varepsilon\log T$ instead of merely $\varepsilon$,
which only adds multiplicative $\log T$ factors to the regret bound in Theorem \ref{thm:adaptive-epsilon}.
With such considerations, for all $j< J$ satisfying $\hat\varepsilon_j\geq\varepsilon$, Lemma \ref{lem:ucb-lcb} and Corollary \ref{cor:rdiff} in the previous proof of Theorem \ref{thm:upper-bound}
would remain valid.

The rest of the proof is quite similar to the proof of Lemma \ref{lem:feasible}, except we have to take into consideration the effect of Step \ref{step:hierarchical-elim} of Algorithm \ref{alg:adaptive-epsilon}.
The proof is again done via induction: at the first epoch $\tau=0$ we have $\mathcal A_j^{(\tau)}=[N]$ and the lemma clearly holds.
Now assume the lemma holds for some $\tau$, we want to prove $S^*\subseteq \mathcal A_j^{(\tau+1)}$ for all $\hat\varepsilon_j\geq\varepsilon$.

Fix arbitrary $i\in S^*\subseteq \mathcal A_j^{(\tau)}$ and assume by way of contradiction that $i\notin \mathcal A_j^{(\tau+1)}$.
Then there exists $k\leq j$ such that $R(S_{\tau,k}^{(i)};\hat v^{(\tau),k})< \gamma_k^{(\tau)} - 2\hat\Delta_{\hat\varepsilon_k}(\tau)$.
Additionally, because $S_{\tau,k}^{(i)}$ is the maximizer of $R(S;\hat v^{(\tau),k})$ for all $|S|\leq k$, $i\in S$, it holds that $R(S_{\tau,k}^{(i)};\hat v^{(\tau),k})\geq R(S^*;\hat v^{(\tau),k})$.
Let also $\hat S_k$ be the assortment attaining $\gamma_k^{(\tau)}$ (i.e., $R(\hat S_k;\hat v^{(\tau),k})=\gamma_k^{(\tau)}$).
Then, invoking Corollary \ref{cor:rdiff}, we have with probability $1-O(NJ/T^2)$ that
\begin{align*}
R(S^*;v)
&\leq R(S^*;\hat v^{(\tau),k})+\hat\Delta_{\hat\varepsilon_k}(\tau)  \leq R(S_{\tau,k}^{(i)};\hat v^{(\tau),k}) + \hat\Delta_{\hat\varepsilon_k}(\tau)\nonumber\\
&< \gamma_k^{(\tau)} - \hat\Delta_{\hat\varepsilon_k}(\tau) = R(\hat S_k;\hat v^{(\tau),k})-\hat\Delta_{\hat\varepsilon_k}(\tau) \leq R(\hat S_k;v) \leq R(S^*;v),
\end{align*}
leading to the desired contradiction.
\end{proof}

\subsection{Proof of Lemma \ref{lem:J-lowerbound}}

\renewcommand{\thelemma}{\ref{lem:J-lowerbound}}
\begin{lemma}[restated]
If $\hat\varepsilon_J\geq \varepsilon$ then with probability $1-O(\tau_0 NJ/T)$,
Algorithm \ref{alg:adaptive-epsilon} will not be re-started.
\end{lemma}

\begin{proof}
We only need to prove that, if $\hat\varepsilon_J\geq\varepsilon$, then for any time period $t$, the condition at step \ref{step:check} of Algorithm \ref{alg:adaptive-epsilon}
is satisfied with probability at most $O(\tau_0 NJ/T^2)$.
Invoking Lemma \ref{lem:feasible-adaptive}, we know that with high probability $S^*\subseteq\mathcal A_j^{(\tau)}$ holds for all $\tau\leq\tau_0$ and $j\leq J$.
Additionally, by algorithm design it is always guaranteed that $\mathcal A_j^{(\tau+1)}\subseteq\mathcal A_k^{(\tau)}$ for any $\hat\varepsilon_k>\hat\varepsilon_j$,
and therefore $\hat S_{\tau,j}^{(i)}\subseteq\mathcal A_j^{(\tau)}$ implies $\hat S_{\tau,j}^{(i)}\subseteq\mathcal A_k^{(\tau)}$. Subsequently, invoking Corollary \ref{cor:rdiff}
we have with probability $O(\tau_0 N/T^2)$ that
\begin{align}
R(\hat S_{\tau,j}^{(i)};\hat v^{(\tau),k})
&\geq R(\hat S_{\tau,j}^{(i)}; v) - \hat\Delta_{\hat\varepsilon_k}(\tau).\label{eq:jlowerbound-1}
\end{align}
Since $i\in\mathcal A_j^{(\tau+1)}$, by the construction of $\mathcal A_j^{(\tau+1)}$ we know that $R(\hat S_{\tau,j}^{(i)};\hat v^{(\tau),j})\geq \gamma_j^{(\tau)}-2\hat\Delta_{\hat\varepsilon_j}(\tau)$.
Let also $\hat S_j$ be the assortment attaining $\gamma_j^{(\tau)}$ (i.e., $R(\hat S_j;\hat v^{(\tau),j})=\gamma_j^{(\tau)}$). Then invoking Corollary \ref{cor:rdiff} again, we have with probability $1-O(\tau_0 N/T^2)$ that
\begin{align}
R(\hat S_{\tau,j}^{(i)};v)& \geq R(\hat S_{\tau,j}^{(i)};\hat v^{(\tau),j}) - \hat\Delta_{\hat\varepsilon_j}(\tau)
\geq \gamma_j^{(\tau)}-3\hat\Delta_{\hat\varepsilon_j}(\tau) = R(\hat S_j;\hat v^{(\tau),j})-3\hat\Delta_{\hat\varepsilon_j}(\tau)\nonumber\\
&\overset{(*)}{\geq} R(S^*;\hat v^{(\tau),j})-3\hat\Delta_{\hat\varepsilon_j}(\tau) \geq R(S^*;v) - 4\hat\Delta_{\hat\varepsilon_j}(\tau).
\label{eq:jlowerbound-2}
\end{align}
Here Eq.~(*) holds because $S^*\subseteq\mathcal A_j^{(\tau)}$.
Combining Eqs.~(\ref{eq:jlowerbound-1}) and (\ref{eq:jlowerbound-2}), we have
\begin{equation}
R(\hat S_{\tau,j}^{(i)};\hat v^{(\tau),k}) \geq R(S^*;v)-\hat\Delta_{\hat\varepsilon_k}(\tau) - 4\hat\Delta_{\hat\varepsilon_j}(\tau) \geq R(S^*;v)-5\hat\Delta_{\hat\varepsilon_k}(\tau),
\label{eq:jlowerbound-3}
\end{equation}
where the last inequality holds because $\hat\Delta_{\hat\varepsilon_j}(\tau)\leq\hat\Delta_{\hat\varepsilon_k}(\tau)$ by definition.
On the other hand, with $\hat S_k$ being the assortment attaining $\gamma_k^{(\tau)}$ (i.e., $R(\hat S_k;\hat v^{(\tau),k})=\gamma_k^{(\tau)}$)
and the fact that $S^*\subseteq\mathcal A_k^{(\tau)}$, invoking Corollary \ref{cor:rdiff} we have with probability $1-O(\tau_0 N/T^2)$ that
\begin{equation}
R(S^*;v) \geq R(\hat S_k;v) \geq R(\hat S_k;\hat v^{(\tau),k})-\hat\Delta_{\hat\varepsilon_k}(\tau) =\gamma_k^{(\tau)}-\hat\Delta_{\hat\varepsilon_k}(\tau).
\label{eq:jlowerbound-4}
\end{equation}
Combining Eqs.~(\ref{eq:jlowerbound-3}) and (\ref{eq:jlowerbound-4}) we have with probability $1-O(\tau_0 N/T^2)$ that
$$
R(\hat S_{\tau,j}^{(i)};\hat v^{(\tau),k}) \geq \gamma_k^{(\tau)}-6\hat\Delta_{\hat\varepsilon_k}(\tau),
$$
which is to be demonstrated.
\end{proof}

\subsection{Proof of Lemma \ref{lem:erj-bound1}}

\renewcommand{\thelemma}{\ref{lem:erj-bound1}}
\begin{lemma}[restated]
For all $j\leq J$ satisfying $\hat\varepsilon_j\geq\varepsilon$, $\mathbb E[\sR(\hat\varepsilon_j)] \lesssim \sum_{\tau\leq\tau_0}\mathbb E[\hat\Delta_{\hat\varepsilon_j}(\tau)\times \wp_jT_\tau]$.
\end{lemma}

\begin{proof}
Because $\hat\varepsilon_j\geq\varepsilon$, by Lemma \ref{lem:feasible-adaptive} we know that $S^*\subseteq\mathcal A_j^{(\tau)}$ for all $\tau\leq\tau_0$ with high probability.
Then, invoking Lemma \ref{lem:regret-single-iter}, the regret incurred by thread $j$ in a single time period is upper bounded by $O(\hat\Delta_{\hat\varepsilon_k}(\tau))$.
Because thread $j$ is sampled with probability $\wp_j$, the expected number of time periods thread $j$ is performed is $\wp_j T$.
This completes the proof of Lemma \ref{lem:erj-bound1}.
\end{proof}

\subsection{Proof of Lemma \ref{lem:erj-bound2}}

\renewcommand{\thelemma}{\ref{lem:erj-bound2}}
\begin{lemma}[restated]
For all $j< J$ satisfying $\hat\varepsilon_j<\varepsilon$ and any $\hat\varepsilon_k>\max\{\hat\varepsilon_j,\varepsilon\}$, it holds that
$
\mathbb E[\sR(\hat\varepsilon_j)] \lesssim \sum_{\tau\leq\tau_0}\mathbb E[\hat\Delta_{\hat\varepsilon_k}(\tau)\times \wp_j T_\tau].
$
\end{lemma}

\begin{proof}
Fix arbitrary $\tau\leq\tau_0$.
According to step \ref{step:check} of Algorithm \ref{alg:active-elimination}, because the value of $J$ does not decrease, we must have
$R(\hat S_{\tau,j}^{(i)};\hat v^{(\tau),k})\geq \gamma_k^{(\tau)}-7\hat\Delta_{\hat\varepsilon_k}(\tau)$ for all assortments $\hat S_{\tau,j}^{(i)}$ explored by thread $j$ in epoch $\tau$.
Because $\hat\varepsilon_k\geq\varepsilon$, we know that $S^*\subseteq\mathcal A_k^{(\tau)}$ with high probability, and using the same argument as in the proof of Lemma \ref{lem:regret-single-iter}
we have with high probability that
$$
R(S^*;v)-R(\hat S_{\tau,j}^{(i)};v)\lesssim \hat\Delta_{\hat\varepsilon_k}(\tau),
$$
which serves as an upper bound of the regret thread $j$ incurs in a single time period it is performed.
Therefore, 
$$
\mathbb E[\sR(\hat\varepsilon_j)] \lesssim \sum_{\tau\leq\tau_0}\mathbb E[\hat\Delta_{\hat\varepsilon_k}(\tau)\times \wp_jT_\tau],
$$
which is to be demonstrated.
\end{proof}

\subsection{Proof of Theorem \ref{thm:adaptive-epsilon}}

Because we restart Algorithm \ref{alg:adaptive-epsilon} whenever $J$ is reduced, and Lemma \ref{lem:J-lowerbound} shows that (with high probability) $\hat\varepsilon_J\leq 2\varepsilon$
always holds. 
Note that it is possible for the value of $\hat\varepsilon_J$ to be far smaller than the actual outlier proportion $\varepsilon$.
In the rest of this section we shall assume without loss of generality that, throughout a consecutive of $T'\leq T$ time periods the value of $J$ does not change,
and furthermore $\hat\varepsilon_J\leq 2\varepsilon$.
The total regret over these $T'$ time periods multiplying $J=O(\log T)$ would then be an upper bound on the total regret over the entire $T$ time periods.

We first consider the regret incurred by thread $j$ with $\hat\varepsilon_j\geq\varepsilon$.
By Lemma \ref{lem:erj-bound1}, the regret incurred by such a thread in epoch $\tau$ can be upper bounded by $O(\mathbb E[\hat\Delta_{\hat\varepsilon_j} T_{\tau,j}])$,
because $T_{\tau,j}=\wp_j T_\tau$.
Replacing $T_\tau$ by $T_{\tau,j}$ in the definition of $\hat\Delta_{\hat\varepsilon_j}(\tau)$, we have that
\begin{equation}
\textstyle
\hat\Delta_{\hat\varepsilon_j}(\tau) \lesssim 
\min\big\{1, \hat\varepsilon_j T/T_{\tau,j}\big\} + \sqrt{\frac{|\mathcal A_j^{(\tau+1)}|\log T}{T_{\tau,j}}} + \frac{|\mathcal A_j^{(\tau+1)}|\log T}{T_{\tau,j}}.
\label{eq:erjb1-eq1}
\end{equation}
On the other hand, because of the sampling protocol in Algorithm \ref{alg:adaptive-epsilon} we have that
\begin{equation}
\mathbb E\left[\sum_\tau T_{\tau,j}\right]  = \wp_j T' \lesssim 2^{-(J-j)} T'.
\label{eq:erjb1-eq2}
\end{equation}
Subsequently, 
\begin{align}
 \mathbb E\sum_\tau\hat\Delta_{\hat\varepsilon_j}(\tau)T_{\tau,j}
\lesssim \mathbb E \sum_\tau \hat\varepsilon_j T_{\tau,j} + \sqrt{NT_{\tau,j}} + N.
\label{eq:erjb1-eq3}
\end{align}
Using Cauchy-Schwarz inequality and the concavity of $f(\cdot)=\sqrt{\cdot}$, we have
\begin{equation}
\mathbb E\sum_\tau \sqrt{T_{\tau,j}}\leq \sqrt{\mathbb E\sum_\tau T_{\tau,j}} \lesssim \sqrt{2^{-(J-j)}{T'}}.
\label{eq:erjb1-eq4}
\end{equation}

Note that $\mathbb E\hat\varepsilon_j T_j = 2^{-j}\times 2^{-(J-j)}T \leq 2^{-J} T \lesssim  \varepsilon T$,
where the last inequality holds because $2^{-J}\lesssim \varepsilon$ thanks to Lemma \ref{lem:J-lowerbound}.
Combining this fact with Eqs.~(\ref{eq:erjb1-eq3},\ref{eq:erjb1-eq4}), we have that
\begin{equation}
 \mathbb E\sum_\tau\hat\Delta_{\hat\varepsilon_j}(\tau)\times \wp_jT_{\tau} \lesssim \varepsilon T +\sqrt{NT} + N \lesssim \varepsilon T+\sqrt{NT}.
\label{eq:final1}
\end{equation}

We next consider regret incurred by threads $j< J$ with $\hat\varepsilon_j<\varepsilon$.
Let $j^*$ be the largest integer such that $\hat\varepsilon_{j^*}\geq\varepsilon$. Then by Lemma \ref{lem:erj-bound2}, the regret incurred by thread $j$ in epoch $\tau$ is upper bounded by 
$$
\mathbb E[\hat\Delta_{\hat\varepsilon_{j^*}}(\tau)\times \wp_jT_\tau] \lesssim \mathbb E[\hat\Delta_{\varepsilon}(\tau)\times \wp_jT_\tau],
$$
where the inequality holds because $\hat\varepsilon_{j^*}\leq2\varepsilon$ by definition.
Using the same analysis in Eqs.~(\ref{eq:erjb1-eq1}),  (\ref{eq:erjb1-eq2}), (\ref{eq:erjb1-eq3}) and (\ref{eq:erjb1-eq4}), we have that
$$
\mathbb E\sum_\tau\hat\Delta_{\varepsilon}(\tau)\times \wp_j T_\tau \lesssim \varepsilon T +\sqrt{2^{(J-j^*)}NT}.
$$
Because $2^{-J}=\hat\varepsilon_J\geq\sqrt{N/T}$ and $2^{-j^*}=\hat\varepsilon_{j^*}\approx\varepsilon$, 
it is easy to verify that $\sqrt{2^{J-j^*}NT} \lesssim \sqrt{\varepsilon}N^{1/4} T^{3/4}$.
Using the inequality $ab\leq (a^2+b^2)/2$ we have that $\sqrt{\varepsilon }N^{1/4}T^{3/4}\leq \varepsilon T + \sqrt{NT}$.

{

\section{Proofs of technical lemmas for Theorem \ref{thm:known-eps-gap-dependent}}

\subsection{Proof of Lemma \ref{lem:known-eps-early-stop}}

\renewcommand{\thelemma}{\ref{lem:known-eps-early-stop}}
\begin{lemma}[restated]
Let $\beta$ be defined in Eq.~(\ref{eq:defn-gap}) and suppose $\beta>0$.
Then with probability $1-O(\tau_0 N/T^2)$, for every $\tau$ satisfying 
\begin{equation}
T_\tau \geq \kappa_0\times \max\left\{\frac{\bar\varepsilon K^2T}{\beta}, \frac{K^2\sqrt{\bar\varepsilon NT\log T}}{\beta},
\frac{K^2 N\log T}{\beta},  \frac{KN\log T}{\beta^2}\right\},
\end{equation}
for some universal constant $\kappa_0>0$, 
it holds that $\mathcal A^{(\tau+1)} = S^*$.
Here $\bar\varepsilon$ is an upper bound estimate of $\varepsilon$.
\end{lemma}
\begin{proof}
Lemma \ref{lem:feasible} has already established that $S^*\subseteq \mathcal A^{(\tau)}$ for all $\tau$ with probability $1-O(\tau_0 N/T^2)$.
Hence we only need to prove that, with probability $1-O(\tau_0 N/T^2)$, any $i\notin S^*$ does not belong to $\mathcal A^{(\tau+1)}$ for $T_\tau$ sufficiently large.

Consider arbitrary $i\notin S^*$ and assume by way of contradiction that $i\in \mathcal A^{(\tau+1)}$.
Define $S_\tau^{*,(i)} := \arg\max_{S\subseteq \mathcal A^{(\tau)},i\in S}R(S_\tau^{*,(i)},v)$,
where $v$ is the underlying true utility parameters.
By definition of the gap parameter $\beta$ and the fact that $S^*\subseteq \mathcal A^{(\tau)}$ (with high probability),
we have that 
\begin{equation}
R(S_\tau^{*,(i)},v) \leq R(S^*,v) - \beta.
\label{eq:proof-known-eps-early-stop-1}
\end{equation}

According to Algorithm \ref{alg:active-elimination}, $i\in \mathcal A^{(\tau+1)}$ means that 
$R(S_\tau^{(i)},\hat v^{(\tau)}) + 2\hat\Delta_{\bar\varepsilon}(\tau) \geq \gamma^{(\tau)}$.
By Corollary \ref{cor:rdiff}, the optimality of $\gamma^{(\tau)},S_\tau^{(i)}$
 and the fact that $\bar\varepsilon\geq\varepsilon$, we have that with probability $1-O(\tau_0 N/T^2)$ that
$\gamma^{(\tau)} \geq R(S^*,v) - \hat\Delta_{\bar\varepsilon}(\tau)$ and 
$R(S_\tau^{(i)},\hat v^{(\tau)}) \leq R(S_\tau^{*,(i)},v) +  \hat\Delta_{\bar\varepsilon}(\tau)$.
Subsequently,
$$
R(S^*,v)-R(S_\tau^{*,(i)},v) - 2\hat\Delta_{\bar\varepsilon}(\tau) \leq \gamma^{(\tau)}-R(S_\tau^{(i)},\hat v^{(\tau)}) \leq 2\hat\Delta_{\bar\varepsilon}(\tau).
$$
Invoking Eq.~(\ref{eq:proof-known-eps-early-stop-1}) again $R(S^*,v)-R(S_\tau^{*,(i)},v) \geq \beta$, we obtain
\begin{equation}
\hat\Delta_{\bar\varepsilon}(\tau)\geq \beta/4.
\label{eq:proof-known-eps-early-stop-2}
\end{equation}
On the other hand, plugging in the definition of $\hat\Delta_{\bar\varepsilon}(\tau)$ we have that, if the condition in Eq.~(\ref{eq:known-eps-early-stop})
holds, $\hat\Delta_{\bar\varepsilon}(\tau)<\beta/4$ contradicting Eq.~(\ref{eq:proof-known-eps-early-stop-2}).
This completes the proof of Lemma \ref{lem:known-eps-early-stop}.
\end{proof}

\subsection{Proof of Theorem \ref{thm:known-eps-gap-dependent}}

Conditioned on the success event in Lemma \ref{lem:known-eps-early-stop} (with probability $1-O(\tau_0 N/T^2)$),
all epochs after the smallest $T_\tau$ satisfying Eq.~(\ref{eq:known-eps-early-stop}) accumulate no regret.
Therefore, to upper bound the cumulative regret of Algorithm \ref{alg:active-elimination},
we can follow the analysis in the proof of Theorem \ref{thm:upper-bound} leading to Eq.~(\ref{eq:final-2})
and replacing $T$ in Eq.~(\ref{eq:final-2}) with $\sum_{\tau\leq \tau^*}T_\tau \lesssim T_{\tau^*}$,
where $\tau^*$ is the smallest epoch whose $T_{\tau^*}$ satisfies Eq.~(\ref{eq:known-eps-early-stop}).
With the condition $\bar\varepsilon\lesssim 1/K^3$, the cumulative regret can be upper bounded by 
\begin{equation}
C_0\times (\bar\varepsilon K^2 T_{\tau^*}\log T + \sqrt{KNT_{\tau^*}\log^3 T}).
\label{eq:proof-known-eps-gap-dependent-1}
\end{equation}

Since $T_{\tau^*}=2T_{\tau^*-1}$ and $T_{\tau^*-1}$ does not satisfy Eq.~(\ref{eq:known-eps-early-stop}), we have that
\begin{equation}
T_{\tau^*} \leq 2\kappa_0\times \max\left\{\frac{\varepsilon K^2T}{\beta}, \frac{K^2\sqrt{\varepsilon NT\log T}}{\beta},
\frac{K^2 N\log T}{\beta},  \frac{KN\log T}{\beta^2}\right\}.
\label{eq:proof-known-eps-gap-dependent-2}
\end{equation}
Combing Eqs.~(\ref{eq:proof-known-eps-gap-dependent-1}) and (\ref{eq:proof-known-eps-gap-dependent-2}), and noting 
that $T_{\tau^*}\leq T$ always holds, the 
regret of Algorithm \ref{alg:active-elimination} can be upper bounded by 
\begin{equation}
C_0'\times \left(\bar\varepsilon K^2 T\log T + \sqrt{\frac{\bar\varepsilon K^3NT\log^3 T}{\beta}} + \frac{\sqrt{K^3 N\log^3 T}(\bar\varepsilon NT\log T)^{1/4}}{\sqrt{\beta}} 
+ \frac{K^2 N\log^2 T}{\beta}\right)
\label{eq:proof-known-eps-gap-dependent-3}
\end{equation}
where $C_0'<\infty$ is a universal constant.
Note that the $\frac{K^2 N\log T}{\beta}$ term in Eq.~(\ref{eq:proof-known-eps-gap-dependent-2}) has been absorbed into the $\frac{KN\log T}{\beta^2}$ term
because $\beta\leq 1$ and therefore $1/\sqrt{\beta}\leq1/\beta$.

Finally we show how the second and the third terms in Eq.~(\ref{eq:proof-known-eps-gap-dependent-3}) are asymptotically dominated
by the first and the fourth terms, thereby proving Theorem \ref{thm:known-eps-gap-dependent}.
For the second term, note that
$$
 \sqrt{\frac{\bar\varepsilon K^3NT\log^3 T}{\beta}}  \leq \frac{1}{2}\left(\bar\varepsilon K^2 T\log T + \frac{KN\log^2 T}{\beta}\right),
$$
thanks to the AM-GM inequality.
For the third term in Eq.~(\ref{eq:proof-known-eps-gap-dependent-3}), invoking the AM-GM inequality again we have
\begin{align*}
 \frac{\sqrt{K^3 N\log^3 T}(\bar\varepsilon NT\log T)^{1/4}}{\sqrt{\beta}} 
 &\leq \frac{1}{2}\left(\sqrt{\bar\varepsilon NK^2T\log^3 T} + \frac{K^2 N\log^2 T}{\beta}\right)\\
 &\leq \frac{1}{2}\left(\frac{1}{2}\left(\bar\varepsilon K^2 T\log T +N\log^2 T \right) + \frac{K^2 N\log^2 T}{\beta}\right)
\end{align*}
This concludes the proof.

\section{Proof of Theorem \ref{thm:adaptive-epsilon-gap}}

We prove Theorem \ref{thm:adaptive-epsilon-gap} by adding the regret incurred on each thread $j$ separately.
First we consider $\mathsf R(\hat\varepsilon_j)$ for those $\hat\varepsilon_j\geq \varepsilon$.
Because $T_j = \wp_j T$, Corollary \ref{cor:erj-bound1-gap} asserts that 
$\mathbb E[\mathsf R(\hat\varepsilon_j)] \lesssim \mathbb E[\sum_{\tau\leq \tau_j^*}\hat\Delta_{\hat\varepsilon_j}(\tau)T_{j,\tau}]$.
Using the same derivation as in the proof of Theorem \ref{thm:known-eps-gap-dependent}, we have that
\begin{align*}
\mathbb E[\mathsf R(\hat\varepsilon_j)]
\lesssim \mathbb E[\hat\varepsilon_j K^2 T_j\log T  + \frac{K^2 N\log^2 T}{\beta}] \lesssim 
\mathbb E[\hat\varepsilon_j T_j + N/\beta] \lesssim \varepsilon T + N/\beta,
\end{align*}
where in the $\lesssim$ notation we omit $\poly(K,\log(NT))$ terms and the last inequality holds because
$\hat\varepsilon_j = 2^{-j}$ and $\wp_j \leq 2\times 2^{-(J-j)}$ and the fact that $\hat\varepsilon_J\leq\varepsilon$ with high probability
(Lemma \ref{lem:J-lowerbound}).

Next consider any $\hat\varepsilon_j$ with $\hat\varepsilon_j\leq\varepsilon$.
Let $k\leq J$ be the largest integer such that $\hat\varepsilon_{k}\geq\varepsilon$, or more specifically
the unique index such that $\hat\varepsilon_{k}\geq\varepsilon>\hat\varepsilon_{k+1}$.
By Corollary \ref{cor:erj-bound2-gap}, we have that
$\mathbb E[\mathsf R(\hat\varepsilon_j)] \lesssim \mathbb E[\sum_{\tau\leq\tau_k^*}\hat\Delta_{\hat\varepsilon_k}(\tau)T_{j,\tau}]$.
By definition, $T_{j,\tau}=2^{j-k}\times T_{k,\tau}$. Subsequently, 
\begin{align*}
\mathbb E[\mathsf R(\hat\varepsilon_j)]
&\lesssim 2^{j-k}\times  \mathbb E[\sum_{\tau\leq\tau_k^*}\hat\Delta_{\hat\varepsilon_k}(\tau)T_{k,\tau}]
\lesssim 2^{j-k}\times \mathbb E[\hat\varepsilon_k T_k + N/\beta]\\
&\lesssim 2^{j-k}\times [2^{-k}\times 2^{-(J-k)}T + N/\beta]
\lesssim 2^{-k} T + 2^{J-k}N/\beta.
\end{align*}
Notice that $2^{-k}\lesssim \varepsilon$ and $2^J \lesssim \sqrt{T/N}$ by definition. Subsequently, 
$$
\mathbb E[\mathsf R(\hat\varepsilon_j)]
\lesssim \varepsilon T + \varepsilon\sqrt{NT}/\beta
\leq \varepsilon T + \frac{1}{2}\left(\varepsilon^2 T + N/\beta^2\right) \lesssim \varepsilon T + N/\beta^2,
$$
where the second inequality holds by the AM-GM inequality.
Theorem \ref{thm:adaptive-epsilon-gap} is thus proved.
The remark after Theorem \ref{thm:adaptive-epsilon-gap} by applying the AM-GM inequality in a different way, or more specifically
$$
 \varepsilon\sqrt{NT}/\beta = \sqrt{\varepsilon^2 NT/\beta^2} \leq \frac{1}{2}\left(\frac{\varepsilon T}{\beta} + \frac{\varepsilon N}{\beta}\right) 
 \lesssim \frac{\varepsilon T}{\beta} + \frac{N}{\beta}.
$$

\section{Proof of Theorem \ref{thm:gap-dependent-lower-bound}}

We first state our adversarial problem instances.
The revenue parameters are set as $r_1=r_2=\cdots=r_N=1$.
For typical customers, the choice model is parameterized by $v_1=v_2=\cdots=v_{K-1}=2/K$,
$v_s=1/K+\Delta$, $v_{i}=1/K$ for all $i\notin\{1,\cdots,K-1,s\}$, where $s\geq K$ is an instance-dependent parameter
and $\Delta\in(0,1-1/K]$ is a small perturbation parameter to be decided later.
Such a problem instance is denoted as $\mathcal I_s$ and the distribution it induces is denoted as $\mathbb P_s$.
For outlier customers, they will \emph{always} purchase the product with the smallest index (i.e., the probability of a no-purchase
for outlier customers is zero).

Because of the additive nature of the lower bound, we only need to prove the $\bih$-regret is lower bounded by 
$\Omega(\min\{\varepsilon T,\sqrt{\varepsilon NT}\})$ and $\Omega(N\log T/(K\beta))$ separately.
For the first $\Omega(\min\{\varepsilon T,\sqrt{\varepsilon NT}\})$ lower bound, one can simply plant a lower bound construction
for gap-free capacitated dynamic assortment planning problems during the first $\lfloor \varepsilon T\rfloor$ time periods.
By the work of \cite{Chen:18tight}, a gap-free lower bound of $\Omega(\min\{T',\sqrt{NT'}\})$ exists,
where $T'$ is the total number of time periods. Plugging in $T'=\lfloor\varepsilon T\rfloor$ we complete the $\Omega(\min\{\varepsilon T,\sqrt{\varepsilon NT}\})$
lower bound in Theorem \ref{thm:gap-dependent-lower-bound}.

In the rest of this proof we focus on the $\Omega(N\log T/(K\beta))$ part of the lower bound, by simply setting $\varepsilon = 0$.
We first compute the sub-optimality gap $\beta$ of the problem instances $\{\mathbb P_i\}$ to determine the appropriate
values of $\Delta$.
\begin{lemma}
For each $\mathbb P_s$, $s\geq K$, the sub-optimality gap $\beta$ defined in Eq.~(\ref{eq:defn-gap})
satisfies $\beta\geq \Delta/8$.
\label{lem:pi-beta}
\end{lemma}
\begin{proof}
Clearly, $S^*=\{1,\cdots,K-1,s\}$ and $S^{*,(i)} = \{1,\cdots,K-1,i\}$ for all $i\notin S^*$, under problem instance $\mathbb P_s$.
Let $R(S)$ denote the expected revenue of assortment $S$ for typical customers. We then have
\begin{align*}
R(S^*) &= \frac{2(K-1)/K + 1/K + \Delta}{1+2(K-1)/K+1/K+\Delta} 
= \frac{2-1/K+\Delta}{3-1/K+\Delta} = 1 - \frac{1}{3-1/K+\Delta}.
\end{align*}
For $S^{*,(i)}$, simply setting $\Delta=0$ we obtain $R(S^{*,(i)}) = 1 - \frac{1}{3-1/K}$. Hence, the sub-optimality gap can be calculated as
$$
\beta = \frac{1}{3-1/K} - \frac{1}{3-1/K+\Delta} = \frac{\Delta}{(3-1/K)(3-1/K+\Delta)} \geq \frac{\Delta}{2\times 4} = \frac{\Delta}{8},
$$
where the last inequality holds because $1/K\leq 1$ and $\Delta\leq 1$.
\end{proof}

From Lemma \ref{lem:pi-beta}, clearly we should set $\Delta = 8\beta$ to make $\{\mathbb P_s:s\geq K\}$ being $\beta$-gap
problem instances. Because $\beta\leq 1/16$, we have that $\Delta\leq 1/2$ and therefore all problem instances are valid for $K\geq 2$.

Our next lemma upper bounds the Kullback-Leibner divergence between $\mathbb P_0$ and $\mathbb P_s$ on certain assortments.
\begin{lemma}
For assortment $S\subseteq[N]$ let $\mathbb P(\cdot|S)$ denote the conditional distribution of purchase activities of typical customers
with offered assortment $S$ and problem instance $\mathbb P$.
Then for any $s\neq s'\in\{K,K+1,\cdots,N\}$ and $S\subseteq[N]$, $|S|\leq K$, $\kl(\mathbb P_s(\cdot|S)\|\mathbb P_{s'}(\cdot|S)) \leq \vct 1\{s\in S\vee s'\in S\}\times 4(K+2)\Delta^2$.
\label{lem:kl-p0-ps}
\end{lemma}
\begin{proof}
First note that if $s\notin S$ and $s'\notin S$ then $\mathbb P_s(\cdot|S)\equiv \mathbb P_{s'}(\cdot|S)$ because the two distributions are exactly the same.
Hence the KL-divergence is zero. In the rest of the proof we assume that either $s\in S$ or $s'\in S$.

Let $a=\sum_{i=1}^{K-1}\vct 1\{i\in S\}$, $b=\vct 1\{s\in S\}$ and $b'=\vct 1\{s'\in S\}$.
Denote $p_i := \mathbb P_s(i|S)$ and $q_i := \mathbb P_{s'}(i|S)$. For the no-purchase action $i=0$ we have
\begin{align*}
p_0 = \frac{1}{1+1 + a/K+b\Delta} \geq \frac{1}{4} \;\;\;\;\text{and}\;\;\;\;
|p_0-q_0| = \left|\frac{1}{2+a/K+b\Delta} - \frac{1}{2+a/K+b'\Delta}\right| \leq \frac{\Delta}{4};
\end{align*}
for $i\notin \{s,s',0\}$, we have
$$
p_i \geq \frac{1/K}{2+a/K+b\Delta} \geq \frac{1}{4K}\;\;\;\;\text{and}\;\;\;\;
|p_i-q_i| \leq \frac{2}{K}\left|\frac{1}{2+a/K+b\Delta}-\frac{1}{2+a/K+b'\Delta}\right| \leq \frac{\Delta}{2K};
$$
for $i\in\{s,s'\}$, we have
$$
p_s \geq \frac{1}{4K}\;\;\;\;\text{and}\;\;\;\;
|p_i-q_i| \leq \frac{\Delta}{2K} + \frac{\Delta}{2+a/K+b'\Delta} \leq \frac{\Delta}{2K} + \frac{\Delta}{2}\leq \Delta.
$$
Invoking Lemma 3 from \citep{Chen:18tight}, we have that
$$
\kl(\mathbb P_0(\cdot|S)\|\mathbb P_s(\cdot|S)) \leq \sum_{i\in S\cup\{0\}}\frac{|p_i-q_i|^2}{p_i} \leq 3(K+1)\Delta^2,
$$
which is to be demonstrated.
\end{proof}

Now consider arbitrary $s\neq s'$ and sufficiently large $T$.
Define $\mathcal T_{s'} := \{t\in[T]: s'\in S_t\}$ as the set of time periods during which product $s'$ is offered in an assortment,
and $T_{s'} := |\mathcal T_{s'}|$.
We then have the following lemma \emph{lower} bounding the $\bih$-regret under $\mathbb P_s$ and $\mathbb P_{s'}$, respectively.
\begin{lemma}
Let $\mathbb E_s$ and $\mathbb E_{s'}$ be expectations taken over the laws of $\mathbb P_s$ and $\mathbb P_{s'}$, respectively.
Suppose $\beta\leq 1/K$. Then
\begin{eqnarray*}
\mathrm{Under}\;\;\;\mathbb P_s, &\quad& \text{$\bih$-Regret } \geq \Omega(\beta)\times \sum_{s'\neq s}\mathbb E_s[T_{s'}];\\
\mathrm{Under}\;\;\;\mathbb P_{s'},&\quad& \text{$\bih$-Regret } \geq \beta(T-\mathbb E_{s'}[T_{s'}]),
\end{eqnarray*}
\label{eq:bih-regret-lowerbound-intermediate}
\end{lemma}
\begin{proof}
The second inequality is obvious from the definition of the sub-optimality gap $\beta$.
To see the first inequality, note that $S^*$ under $\mathbb P_s$ is $\{1,2,\cdots,K-1,s\}$.
If $S_t=\{1,2,\cdots,K-1,s'\}$ for some $s'\neq s$, then this assortment suffers an instantaneous regret of $\beta$;
if $S_t$ contains another $s''\notin \{s,s',1,\cdots,K-1\}$, then this assortment suffers an additional instantaneous regret of $\Omega(1/K)$
since one product from the first $(K-1)$ products must be missed.
Because $\beta\leq 1/K$, the lemma is proved.
\end{proof}

Let $f_T\leq T/2$ and $c_T>0$ be parameters to be decided later.
We will select $f_T$ such that, if the conclusion in Theorem \ref{thm:gap-dependent-lower-bound} holds, then
the $\bih$-regret under $\mathbb P_{s'}$ is at most $\beta f_T$, or more specifically $T-\mathbb E_{s'}[T_{s'}]\leq f_T$.
By Markov's inequality, this implies that
\begin{equation}
\mathbb P_{s'}\big[T_{s'}\leq f_T\big] = \mathbb P_{s'}\big[T-T_{s'}>T-f_T\big] \leq \frac{T-\mathbb E_{s'}[T_{s'}]}{T-f_T} \leq \frac{2f_T}{T}. 
\label{eq:proof-gap-dependent-lb-1}
\end{equation}

Next, let $\{i_t\}_{t\in\mathcal T_{s'}}$ be the purchase activities realized during time periods at which product $s'$ is offered as part of the assortment.
Define also the log-likelihood ratio $L_{\mathcal T_{s'}} := \sum_{t\in\mathcal T_{s'}}\log\frac{\mathbb P_s(i_t|S_t)}{\mathbb P_{s'}(i_t|S_t)}$.
We then have that, for any event $A\subseteq\{T_{s'}=\tau\}$, 
\begin{align*}
\mathbb P_{s'}[A] &= \int_A\prod_{t\in\mathcal T_{s'}}\frac{\mathbb P_{s'}(i_t|S_t)}{\mathbb P_s(i_t|S_t)}\ud \mathbb P_s
= \int_A\exp\left\{\sum_{t\in\mathcal T_{s'}}\log\frac{\mathbb P_{s'}(i_t|S_t)}{\mathbb P_s(i_t|S_t)}\right\}\ud\mathbb P_s
= \int_A\exp\left\{-L_{\mathcal T_{s'}}\right\}\ud\mathbb P_s.
\end{align*} 
Hence, if $A\subseteq\{T_{s'}=\tau\wedge L_{\mathcal T_{s'}}\leq c_T\}$, then $\mathbb P_{s'}[A]\geq e^{-c_T}\mathbb P_s[A]$,
or equivalently $\mathbb P_s[A]\leq e^{c_T}\mathbb P_{s'}[A]$.
Using the law of total probability we have the following:
\begin{align}
\mathbb P_s\big[T_{s'}\leq f_T\big]
&= \mathbb P_s\big[T_{s'}\leq f_T\wedge \mathcal L_{\mathcal T_{s'}}\leq c_T\big] + \mathbb P_s\big[T_{s'}\leq f_T\wedge \mathcal L_{\mathcal T_{s'}}> c_T\big]\nonumber\\
&\leq e^{c_T}\mathbb P_{s'}\big[T_{s'}\leq f_T\wedge \mathcal L_{\mathcal T_{s'}}\leq c_T\big] +  \mathbb P_s\big[T_{s'}\leq f_T\wedge \mathcal L_{\mathcal T_{s'}}> c_T\big]\nonumber\\
&\leq e^{c_T}\mathbb P_{s'}\big[T_{s'}\leq f_T\big] +  \mathbb P_s\big[\mathcal L_{\mathcal T_{s'}}> c_T|T_{s'}\leq f_T\big].
\label{eq:proof-gap-dependent-lb-2}
\end{align}

Now set $f_T=c_0''\log T/(K\beta^2)$ for some sufficiently small universal constant $c_0''>0$ and $c_T=0.5\log T$.
If the $\bih$-regret under $\mathbb P_{s'}$ exceeds $\beta f_T$ then we have already proved Theorem \ref{thm:gap-dependent-lower-bound}.
Otherwise, by Eq.~(\ref{eq:proof-gap-dependent-lb-1}) we have that
\begin{equation}
e^{c_T}\mathbb P_{s'}\big[T_{s'}\leq f_T\big]  
\leq \exp\{0.5\log T + \log(2f_T) - \log T\} = 2f_T/\sqrt{T} = o(1).
\label{eq:proof-gap-dependent-lb-3}
\end{equation}
On the other hand, recall the definition that $\mathcal L_{\mathcal T_{s'}} = \sum_{t\in\mathcal T_{s'}}\log\frac{\mathbb P_s(i_t|S_t)}{\mathbb P_{s'}(i_t|S_t)}$,
we know that $\mathcal L_{\mathcal T_{s'}}$ is the partial sum of a martingale and furthermore
$\mathbb E_{s}[\mathcal L_{\mathcal T_{s'}}] \leq \mathbb E_s[T_{s'}]\times \max_{S\subseteq[N],|S|\leq K}\kl(\mathbb P_s(\cdot|S)\|\mathbb P_{s'}(\cdot|S))
\leq \mathbb E_s[T_{s'}]\times O(K\Delta^2) =  \mathbb E_s[T_{s'}]\times O(K\beta^2)$,
thanks to Lemma \ref{lem:kl-p0-ps}.
By setting $c_0''$ to be sufficiently small, we have that $c_T-\mathbb E[\mathcal L_{\mathcal T_{s'}}|T_{s'}\leq f_T]
\geq 0.5\log T - O(f_TK\beta^2) = \Omega(\log T) = \Omega(K\beta^2)\times f_T$.
Because both $K$ and $\beta$ are constants not changing with $T$, by the law of large numbers we have that
\begin{equation}
\mathbb P_s\big[\mathcal L_{\mathcal T_{s'}}> c_T|T_{s'}\leq f_T\big] = o(1).
\label{eq:proof-gap-dependent-lb-4}
\end{equation}

Combining Eqs.~(\ref{eq:proof-gap-dependent-lb-2},\ref{eq:proof-gap-dependent-lb-3},\ref{eq:proof-gap-dependent-lb-4}) we have that,
for every $s'\neq s$, 
$$
\mathbb P_s\big[T_{s'}<f_T\big] = o(1).
$$
Using Markov's inequality, 
$$
\mathbb E_s[T_{s'}] \geq \mathbb P_s\big[T_{s'}\geq f_T\big]\times f_T = \Omega(f_T).
$$
Summing over all $s'\neq s$ we obtain
\begin{align*}
\text{$\bih$-regret under $\mathbb P_s$} &\geq \Omega(\beta)\times \sum_{s'\neq s}\mathbb E_s[T_{s'}]
\geq \Omega(\beta\times (N-K)\times f_T) \\
&= \Omega(\beta N\times \log T/(K\beta^2)) = \Omega(N\log T/(K\beta)),
\end{align*}
which is to be demonstrated.
}

\section{Tail inequalities}

\renewcommand{\thelemma}{9}
\begin{lemma}[Bernstein's inequality for martingale process \cite{freedman1975tail}]
Let $X_1,\cdots,X_n$ be centered random variables satisfying $|X_i|\leq M$ almost surely for all $i$,
and that $\sum_{i\leq s}X_i$ for $s\leq n$ forms a martingale process.
Then for any $t>0$,
$$
\Pr\left[\bigg|\sum_{i=1}^n X_i\bigg|>t\right] \leq 2\exp\left\{-\frac{t^2/2}{\sum_i\mathbb E[X_i^2] + Mt/3}\right\}.
$$
As a corollary, for any $\delta>0$, 
$$
\Pr\left[\bigg|\sum_{i=1}^n X_i\bigg| > \frac{2}{3}M\log (1/\delta) + \sqrt{2V^2\log (1/\delta)}\right] \leq 2\delta,
$$
where $V^2 = \sum_i \mathbb E[X_i^2]$.
\label{lem:bernstein-tech}
\end{lemma}

\end{document}